\documentclass[lettersize,journal]{IEEEtran}
\usepackage{amsmath,amsfonts}
\usepackage{algorithmic}
\usepackage{algorithm}
\usepackage{array}
\usepackage[caption=false,font=footnotesize,labelfont=rm,textfont=rm]{subfig}  
\usepackage{textcomp}
\usepackage{stfloats}
\usepackage{url}
\usepackage{verbatim}
\usepackage{graphicx}
\usepackage{cite}
\usepackage{pifont}
\usepackage{booktabs}
\usepackage{multirow}
\usepackage{xcolor}
\usepackage{amsthm}
\usepackage[colorlinks,
linkcolor=red,
anchorcolor=blue,
citecolor=blue]{hyperref}

\newtheorem{theorem}{Theorem}
\newtheorem{lemma}{Lemma}
\newtheorem{corollary}{Corollary}

\begin{document}

\title{HI-GVF: Shared Control based on Human-Influenced Guiding Vector Fields for Human-multi-robot Cooperation}

\author{Pengming Zhu, Zongtan Zhou, Weijia Yao, Wei Dai, Zhiwen Zeng, and Huimin Lu
\thanks{This work was supported by the National Science Foundation of China under Grant 62203460, U22A2059, Major Project of the Natural Science Foundation of Hunan Province (No. 2021JC0004).}
\thanks{Pengming Zhu, Zongtan Zhou, Wei Dai, Zhiwen Zeng, Huimin Lu are with the College of Intelligence Science and Technology, National University of Defense Technology, Changsha 410073, China (e-mail: zhupengming@nudt.edu.cn; narcz@163.com; weidai\_nudt@foxmail.com; zengzhiwen@nudt.edu.cn; lhmnew@nudt.edu.cn).}
\thanks{Weijia Yao is with School of Robotics, Hunan University, Changsha 410082, China (e-mail: wjyao@hnu.edu.cn).}}

%

\maketitle

\begin{abstract}	
Human-multi-robot shared control leverages human decision-making and robotic autonomy to enhance human-robot collaboration. While widely studied, existing systems often adopt
a leader-follower model, limiting robot autonomy to some extent.
Besides, a human is required to directly participate in the motion control of robots through teleoperation, which significantly burdens the operator.
To alleviate these two issues, we propose a layered shared control computing framework using human-influenced guiding vector fields (HI-GVF) for human-robot collaboration. 
HI-GVF guides the multi-robot system along a desired path specified by the human. Then, an intention field is designed to merge the human and robot intentions, accelerating
the propagation of the human intention within the multi-robot system.
Moreover, we give the stability analysis of the proposed model and use collision avoidance based on safety barrier certificates to fine-tune the velocity. Eventually, considering the firefighting task as an example scenario, we conduct simulations and experiments using multiple human-robot interfaces (brain-computer interface, myoelectric wristband, eye-tracking), and the results demonstrate that our proposed approach boosts the effectiveness and performance of the task.
\end{abstract}


\section{Introduction}

\IEEEPARstart{M}{ulti-robot} systems (MRS) have attracted much attention from the research community, and they are widely used in tasks such as monitoring \cite{smith2017multi}, firefighting \cite{saez2010multi}, search and rescue \cite{liu2013robotic}, source seeking \cite{fabbiano2018distributed}, etc. 
Despite that considerable efforts have been made by researchers to enhance the autonomy of distributed multi-robot systems, there are still several challenges due to, e.g., perception inaccuracies, communication delays, and limited decision-making capabilities of robots \cite{dai2019task}. 
Therefore, practical multi-robot systems often require human operators' supervision and intervention, which can help assess and find potential targets, identify threats, react to emergencies and especially, provide intelligent guidance when planning algorithms fail. 
In many cases, human partners can operate remotely using the shared control technology to influence multi-robot systems, which can perform collaborative tasks with guaranteed safety autonomously when human influence is absent and thus alleviate the physical and mental workload on humans.

Shared control methods used in remote operating systems are usually based on mixed inputs \cite{saleh2013shared,krzysiak2022information} or haptic feedback \cite{zhu2023a,sieber2015multi}.
There are various human-multi-robot interaction systems.
For instance, a task-dependent graphical user multi-touch interface \cite{ayanian2014controlling,patel2019mixed}, a gesture-based interface \cite{podevijn2014dgesturing,suresh2020human} of a mobile handheld device \cite{patel2019mixed}, and a multi-modal interface with gesture, voice, and vision \cite{gromov2016wearable}.
As noted in several reviews \cite{chen2014human,kolling2015human}, the effectiveness of such a team is primarily
influenced by the type of information operators receive from the robots and the tools they have available to exert control.
In some cases, multiple modalities are necessary for the expression of human intentions. Assume a scenario similar to Fig. \ref{background}, where an operator and a group of robots collaborate on a mission together, where the robots perform tasks in the rubble, and the operator provides a more comprehensive view via image transmission from the UAV. Humans hold remote operating equipment/tools and, therefore, need other methods to send commands to robots, such as covert hand gestures and smart helmets that can capture EEG signals as well as eye movement information. In this way, the operator can provide the robots with more information about the environment, where the target is (e.g., threats that need to be removed, trapped persons), and, if necessary, the operator even needs to be given a certain trajectory that will lead the robot to a certain area.

\begin{figure}
	\centering
	\includegraphics[width=\linewidth]{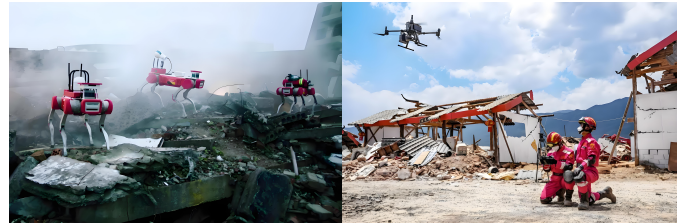}
	\caption{Human and multi-robot collaboration.}
	\label{background}
\end{figure}

Based on this motivation, this paper proposes a layered shared control computing framework based on \textbf{h}uman-\textbf{i}nfluenced \textbf{g}uiding \textbf{v}ector \textbf{f}ields (HI-GVF) for multi-robot systems, which is shown in Fig. \ref{framework}.
The blocks of different colors represent different research contents.
Robots acquire different world models by sensors for calculation and decision-making, as shown in Fig.~\ref{framework} (gray block). 

Human operators acquire the world models not based on direct observations of the environment but on a virtual environment based on remote images.
This method gives humans a broader view of the situation, which is useful for integrated planning.
The operator can then give human intention by plotting a trajectory for robots or selecting a priority target area, as shown in Fig. \ref{framework} (green block).
Robot intention and human intention are fused through a shared control framework, as shown in Fig. \ref{framework} (blue block).
The components are explained in detail in subsequent sections.

The main contributions of this paper are as follows:
\begin{itemize}
	\item [1)]We propose a shared control framework based on HI-GVF for human-multi-robot cooperation, without requiring humans to intervene at every moment, and the operator can intuitively give a certain desired path as an input to guide robots.
	
	\item [2)]We design a layered shared control framework, where an intention field is utilized to merge the robot intention generated by the robot local controller and the human intention generated by HI-GVF in the upper layer, and a policy-blending model is used to fuse various VFs in the lower layer.
	\item [3)]We give a stability analysis of the proposed method and use an obstacle avoidance method based on optimization in a minimally invasive way, which ensures stability in obstacle avoidance for the whole formation.
	
	\item [4)]We conduct simulations and physical experiments using multiple human-robot interfaces (brain-computer interface, myoelectric wristband, eye-tracking) based on self-designed human-multi-robot interaction system to verify
	the effectiveness of the proposed method.
\end{itemize}

\begin{figure*}
	\centering
	\includegraphics[width=0.8\linewidth]{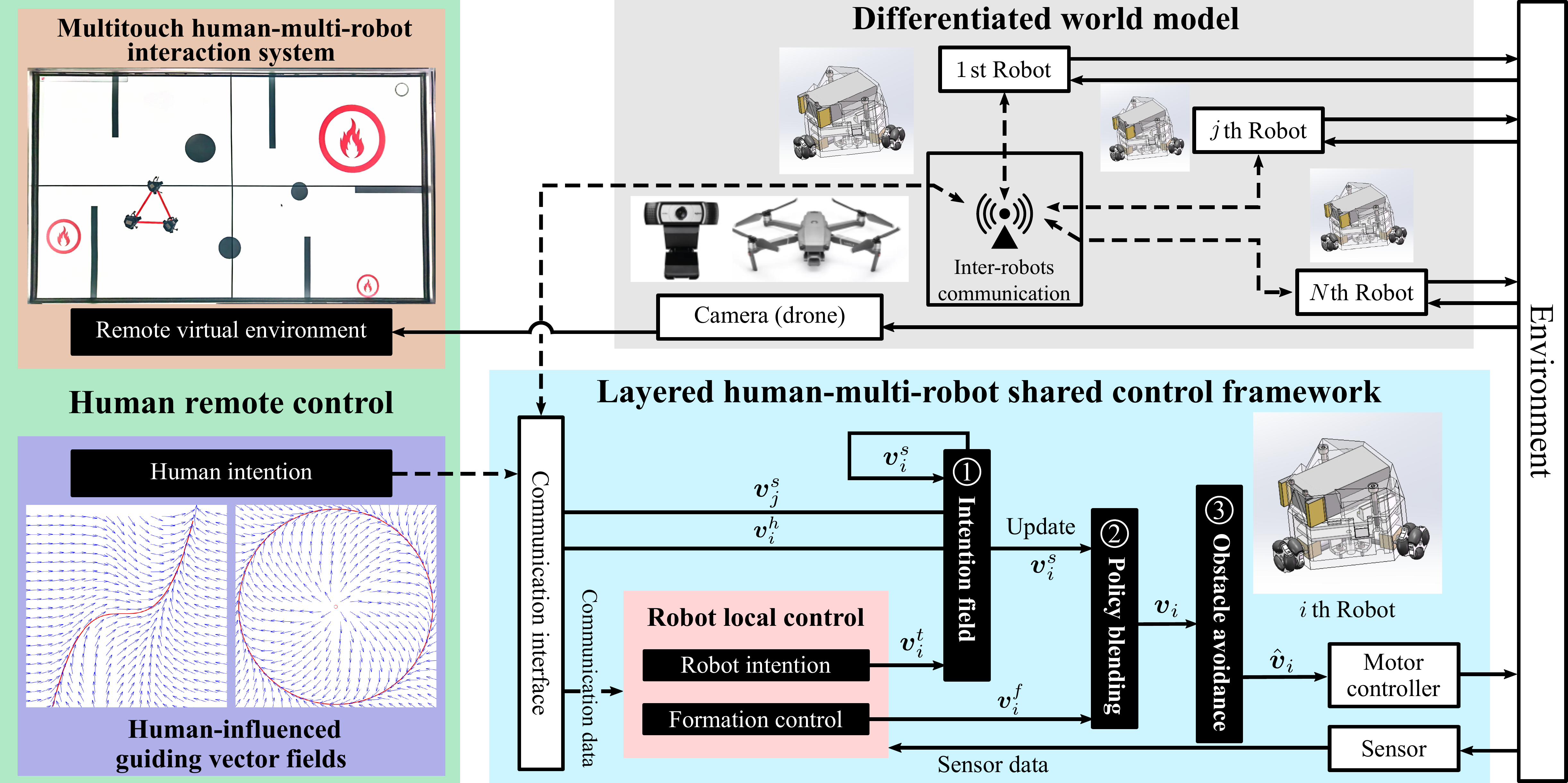}
	\caption{The framework of the proposed human-influenced guiding vector field shared control method. The human operator and each robot are independent units, ultimately acting on the MRS through shared control. Different colored blocks are used to distinguish the different research contents in this figure.}
	\label{framework}
\end{figure*}

\section{preliminary: guiding vector field}
\label{sec2}

	In this section, we introduce preliminary knowledge about guiding vector field.
	Then, we utilize HI-GVF to describe human intention, which is used to influence the collective behavior of the multi-robot system. The HI-GVF is derived based on a desired path given by the human operator, who expects a robot to follow the desired path for a specified period of time according to the human intention.

\subsection{GVF}
To derive the guiding vector field, we first suppose that the desired path can be described by an implicit function:
\begin{equation}
\mathcal{P}=\{\xi \in \mathbb{R}^2:\phi(\xi)=0 \},
\label{eq_desired_path}
\end{equation}
where $\phi:\mathbb{R}^2 \to \mathbb{R}$ is twice continuously differentiable, and $\xi=(x,y) \in \mathbb{R}^2$ represents the two-dimensional position.
Then a guiding vector field proposed in \cite{kapitanyuk2017guiding} to solve the path-following problem is defined by:
\begin{equation}
\mathcal{X}(\xi)=\gamma E \nabla \phi(\xi)-k \phi(\xi) \nabla \phi(\xi),
\label{vector field_phi}
\end{equation}
where $E$ is the $90^{\circ}$ rotation matrix 
$ \left[ \begin{smallmatrix}
0 & -1\\
1 & 0
\end{smallmatrix} \right]$, $\gamma \in \{1,-1\}$ is used to specify the propagation direction along the desired path, $k$ is an adjustable parameter. 
The first term of \eqref{vector field_phi} is the \emph{tangential component}, enabling the robot to move along the desired path, while the second term of \eqref{vector field_phi} is the \emph{orthogonal component}, assisting the robot to move closer to the desired path.
	Therefore, intuitively, the guiding vector field can guide the robot to move toward and along the desired path at the same time.


\begin{figure}[h!]
	\centering
	\includegraphics[width=0.93\linewidth]{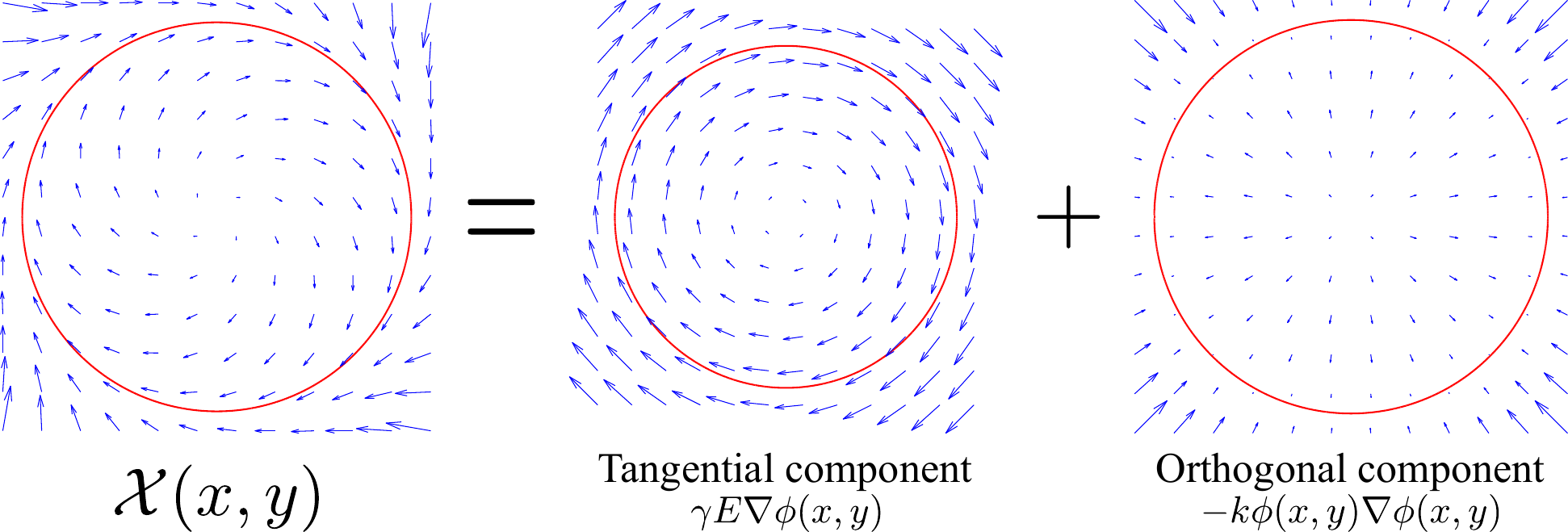}
	\caption{The visualization of a guiding vector field. The corresponding desired path (red curve) is the circle described by choosing $\phi(x,y) = x^2 + y^2 - R^2$, where $R$ is the circle radius, in (1). Each blue arrow represents a vector of the corresponding vector field at the position}.
	\label{VF}
\end{figure}


\subsection{HI-GVF}
When the desired path is occluded by obstacles, a robot needs to deviate from the desired path to avoid obstacles, and then return to the desired path.
To address this problem, we design a smooth composite vector field, consisting of vector fields generated by the desired path and obstacles and integrated with two bump functions \cite{yao2022guiding}.

\subsubsection{Reactive boundary and repulsive boundary}
Regardless of the shapes (circle or bar) of the obstacles, we design some boundaries to enclose each obstacle, such that avoiding collisions with the obstacles is simplified as avoiding collisions with the boundaries, as shown in Fig. \ref{QR}(a).
The reactive boundary $\mathcal{R}$ and the repulsive boundary $\mathcal{Q}$ around the circle obstacle or bar barrier are defined as follows:
\begin{equation}
\mathcal{R}_i=\{\xi \in \mathbb{R}^2: \varphi_i(\xi)=0 \},
\label{eq_Reactive}
\end{equation}
\begin{equation}
\mathcal{Q}_i=\{\xi \in \mathbb{R}^2: \varphi_i(\xi)=c_i \},
\label{Q-smooth}
\end{equation}
for $i \in \mathcal{I}$, where $\mathcal{I}$ represents the set of a finite number of obstacles, $\varphi_i:\mathbb{R}^2\to\mathbb{R}$ is a twice continuously differentiable function, and $c_i \neq 0$ is a given constant.
The reactive boundary $\mathcal{R}$ divides the workspace into the interior, denoted by $^\text{in}\mathcal{R}$, and the exterior, denoted by $^\text{ex}\mathcal{R}$, as shown in Fig. \ref{QR}(b).
For convenience, $^\text{in}\mathcal{R}$ is also called the reactive area, and $^\text{ex}\mathcal{R}$ is the non-reactive area.
When the robot enters the reactive area, it is able to detect and react to the obstacle.
Similarly, we define $^\text{in}\mathcal{Q}$ and $^\text{ex}\mathcal{Q}$ to represent the repulsive area and the non-repulsive area, respectively, as shown in Fig. \ref{QR}(c).
The repulsive area is a region where the robot is forbidden to enter; otherwise collision with the obstacle would happen.
Besides, we use $\overline{(\cdot)}$ to represent the area containing the corresponding boundary (i.e., the closure of a given set).
For instance, $\overline{^\text{in}\mathcal{R}}$ denotes the reactive area $^\text{in}\mathcal{R}$ containing the reactive boundary $\mathcal{R}$.
\begin{figure}[h!]
	\centering
	\subfloat[Desired path $\mathcal{P}$, and the obstacle (black disk) with reactive boundary $\mathcal{R}$ and repulsive boundary $\mathcal{Q}$]{\includegraphics[width=0.9\linewidth]{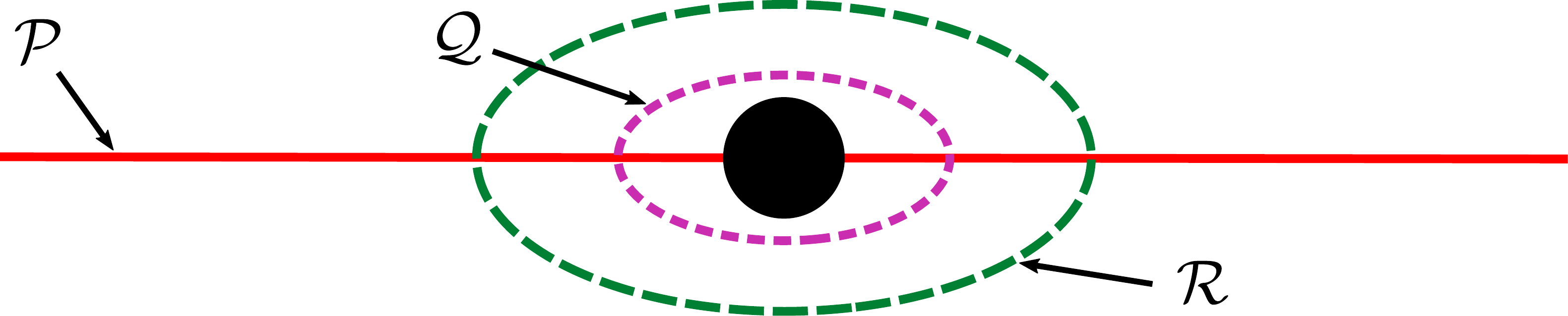}}\\
	\label{QR_a}
	\subfloat[Reactive area $^\text{in}\mathcal{R}$ (green elliptic disk) and non-reactive area $^\text{ex}\mathcal{R}$]{\includegraphics[width=0.45\linewidth]{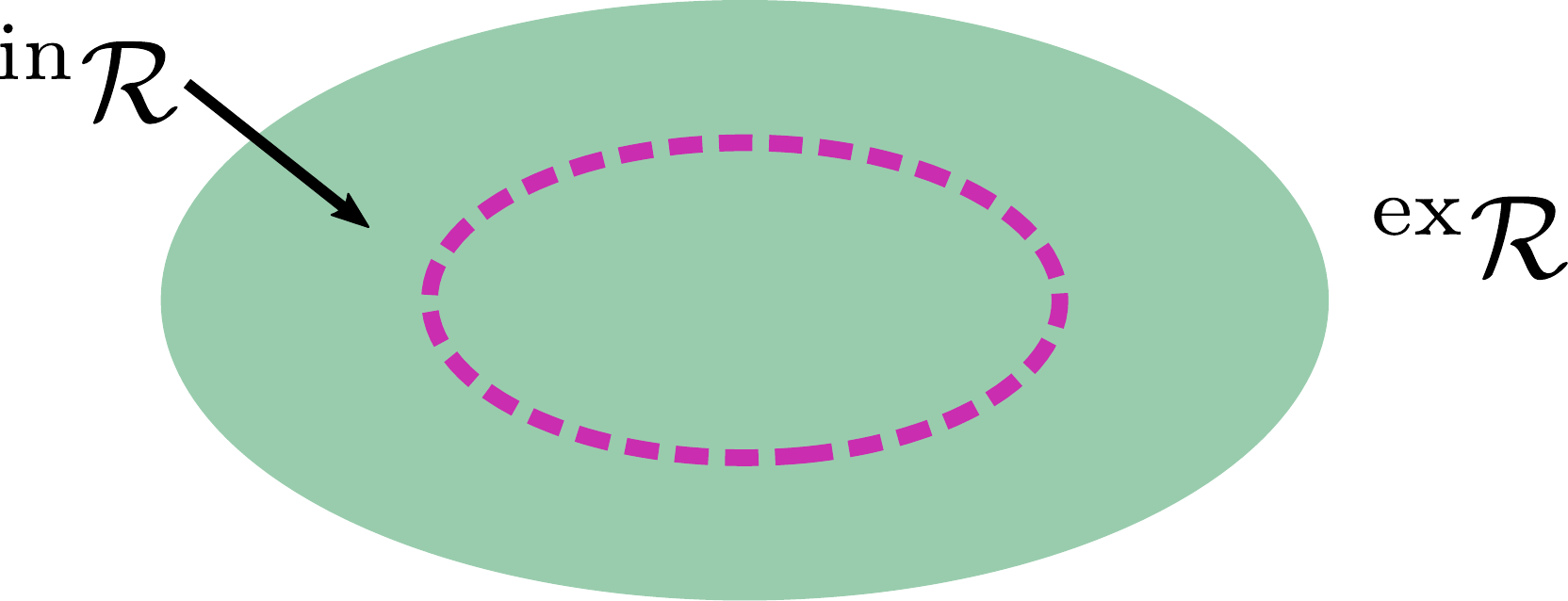}}
	\label{QR_b}
	\hspace{5mm}
	\subfloat[Repulsive area $^\text{in}\mathcal{Q}$ (pink elliptic disk) and non-repulsive area $^\text{ex}\mathcal{Q}$]{\includegraphics[width=0.45\linewidth]{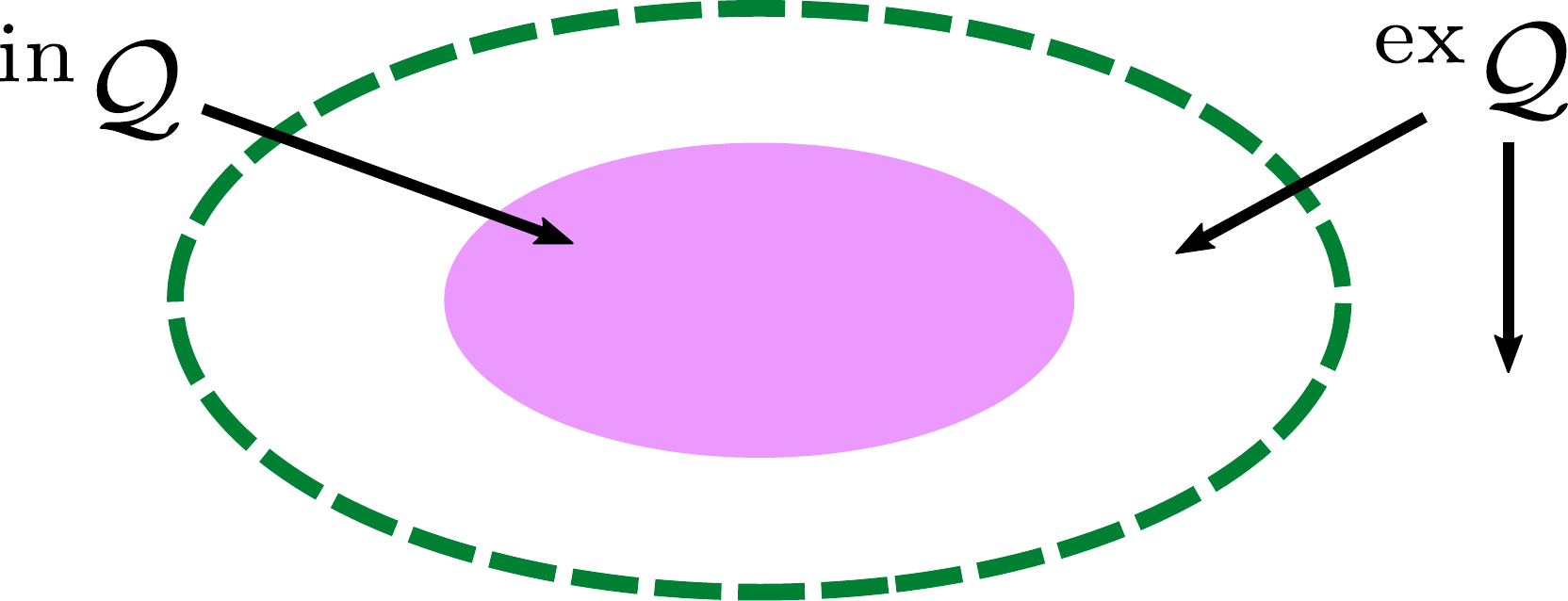}}
	\label{QR_c}
	\caption{Illustration of the reactive and repulsive boundary and area of one obstacle.} 
	\label{QR}
\end{figure}


	For the desired path in \eqref{eq_desired_path} and the reactive boundary in \eqref{eq_Reactive}, following \eqref{vector field_phi}, we can obtain the corresponding vector fields $\mathcal{X}_\mathcal{P},\mathcal{X}_{\mathcal{R}_i}$ with the constants $\gamma_p,\gamma_i, k_p,k_{r_i}$ sharing the same meanings as those in \eqref{vector field_phi}.

\subsubsection{Zero-in and zero-out bump function}
To smooth the motion of the robot as it passes through the reactive area where a robot is able to sense the obstacle, we use two smooth bump functions to generate a composite vector field so that the velocity does not change abruptly near the reactive boundary.

For a reactive boundary $\mathcal{R}$ and a repulsive boundary $\mathcal{Q}$, we can define smooth bump function $\sqcap_{\mathcal{R}_i}$, $\sqcup_{\mathcal{Q}_i}$ as follows:

\begin{equation}
\sqcap_{\mathcal{R}_i}(\xi)=\left\{\begin{array}{ll}
1 & \xi \in \overline{^\text{in}\mathcal{Q}_i} \\
Z_i(\xi) & \xi \in  {^\text{ex}\mathcal{Q}_i} \cap {^\text{in}\mathcal{R}_i} \\
0 & \xi \in \overline{^\text{ex}\mathcal{R}_i}
\end{array}\right.,
\label{zero-out function}
\end{equation}

\begin{equation}
\sqcup_{\mathcal{Q}_i}(\xi)=\left\{\begin{array}{ll}
0 & \xi \in \overline{^\text{in}\mathcal{Q}_i} \\
S_i(\xi) & \xi \in  {^\text{ex}\mathcal{Q}_i} \cap {^\text{in}\mathcal{R}_i} \\
1 & \xi \in \overline{^\text{ex}\mathcal{R}_i}
\end{array}\right.,
\label{zero-in function}
\end{equation}
where $Z_i:{^\text{ex}\mathcal{Q}_i} \cap {^\text{in}\mathcal{R}_i} \rightarrow(0,1)$ and $S_i:{^\text{ex}\mathcal{Q}_i} \cap {^\text{in}\mathcal{R}_i} \rightarrow(0,1)$ are smooth functions. 

\begin{figure}[h!]
	\centering
	\subfloat[zero-out bump function $\sqcap_{\mathcal{R}_i}$]{\includegraphics[width=0.49\linewidth]{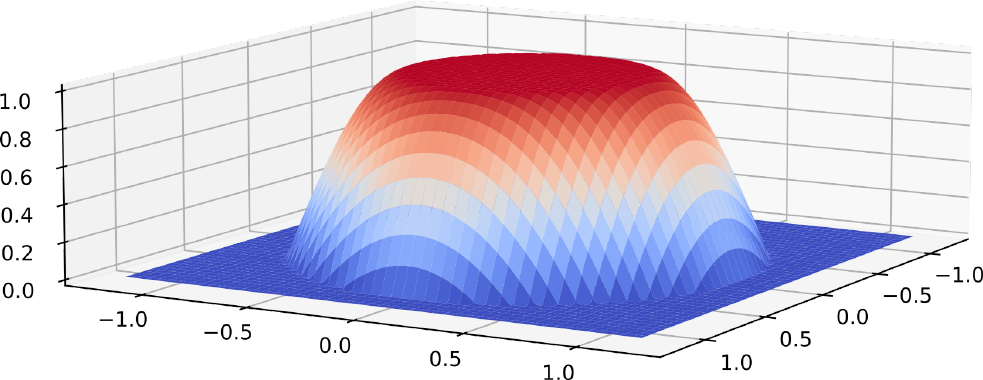}}
	\label{zero-out}
	\subfloat[zero-in bump function $\sqcup_{\mathcal{Q}_i}$]{\includegraphics[width=0.49\linewidth]{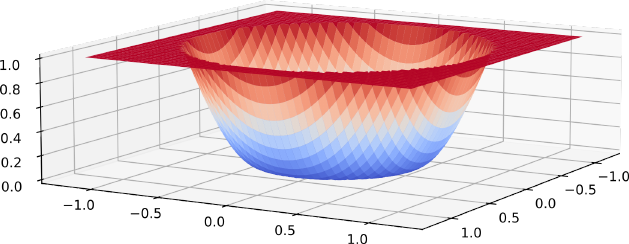}}
	\label{zero-in}
	\caption{Illustration of smooth zero-out and zero-in bump functions.} 
	\label{bump_functions}
\end{figure}

Intuitively, we call $\sqcap_{\mathcal{R}_i}$ (see Fig. \ref{bump_functions}(a)) a smooth zero-out function with respect to the reactive boundary $\mathcal{R}_i$, and $\sqcup_{\mathcal{Q}_i}$ (see Fig. \ref{bump_functions}(b)) a smooth zero-in function with respect to the repulsive boundary ${\mathcal{Q}_i}$. Evidently, $Z_i(\xi)$ approaches smoothly value 1 when $\xi$ approaches $\mathcal{Q}_i$, and approaches smoothly value 0 when $\xi$ approaches $\mathcal{R}_i$, and the converse applies for $S_i(\xi)$.

\subsubsection{Composite vector field}
Using the smooth zero-out and zero-in bump functions to blend two vector fields, we can obtain a composite vector field $\mathcal{X}_c$ as follows:
\begin{equation}
\mathcal{X}_c(\xi)=\left(\prod_{i \in \mathcal{I}} \sqcup_{\mathcal{Q}_i}(\xi)\right) \hat{\mathcal{X}}_{\mathcal{P}}(\xi)+\sum_{i \in \mathcal{I}}\left(\sqcap_{\mathcal{R}_i}(\xi) \hat{\mathcal{X}}_{\mathcal{R}_i}(\xi)\right),
\label{composite VF}
\end{equation}
where $\hat{(\cdot)}$ is the normalization notation (i.e., $\hat{\boldsymbol{v}}=\boldsymbol{v}/ \Vert \boldsymbol{v} \Vert$ for a non-zero vector). In our paper, it is assumed that different reactive areas do not overlap, which is a common assumption; if they overlap, then one can redefine the reactive boundaries such that this assumption is satisfied (e.g., by shrinking the sizes of the boundaries or by merging two boundaries into a large one). Therefore, according to \eqref{zero-out function} and \eqref{zero-in function}, \eqref{composite VF} is equivalent to: 
\begin{equation}
\mathcal{X}_c(\xi) =\left\{\begin{array}{ll}
\hat{\mathcal{X}}_{\mathcal{R}_i}(\xi) & \xi \in \overline{^\text{in} \mathcal{Q}_i}  \\
S_i(\xi) \hat{\mathcal{X}}_{\mathcal{P}}(\xi)+Z_i(\xi) \hat{\mathcal{X}}_{\mathcal{R}_i}(\xi) & \xi \in {^\text{ex} \mathcal{Q}_i} \cap { ^\text{in}} \mathcal{R}_i  \\
\hat{\mathcal{X}}_{\mathcal{P}}(\xi) & \xi \in \overline{^\text{ex} \mathcal{R}_i}
\end{array}\right..
\label{com VF}
\end{equation}
\begin{figure}[h!]
	\centering
	\label{Xc_VF_a}
	\subfloat[VF $\hat{\mathcal{X}}_\mathcal{P} $ for desired path $\mathcal{P}$]{\includegraphics[width=0.485\linewidth]{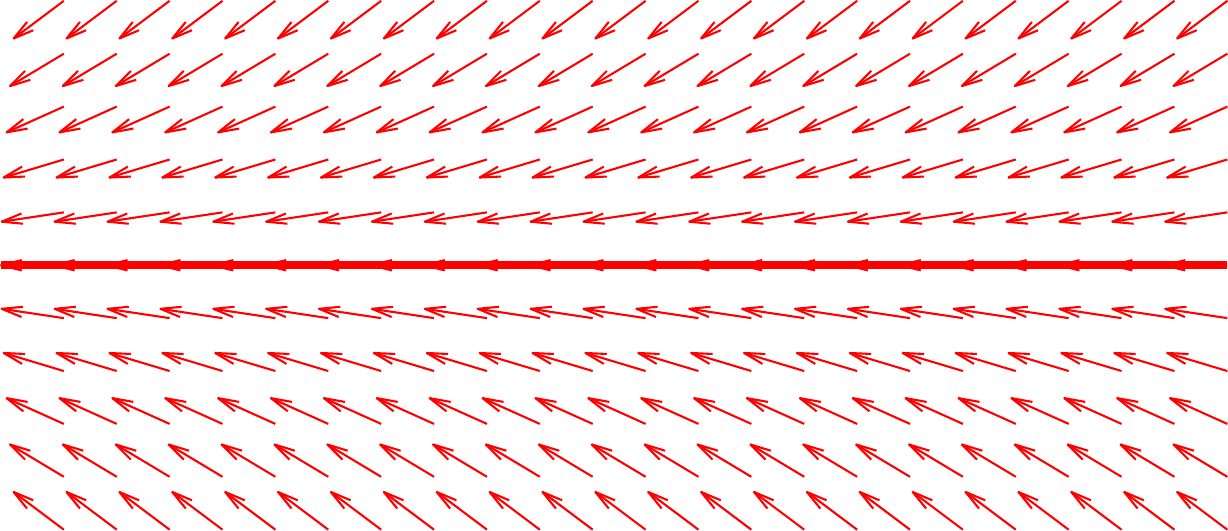}}
	\label{Xc_VF_b}
	\hspace{0.4mm}
	\subfloat[VF $\hat{\mathcal{X}}_{\mathcal{R}_i} $ for reactive boundary $\mathcal{R}_i$]{\includegraphics[width=0.485\linewidth]{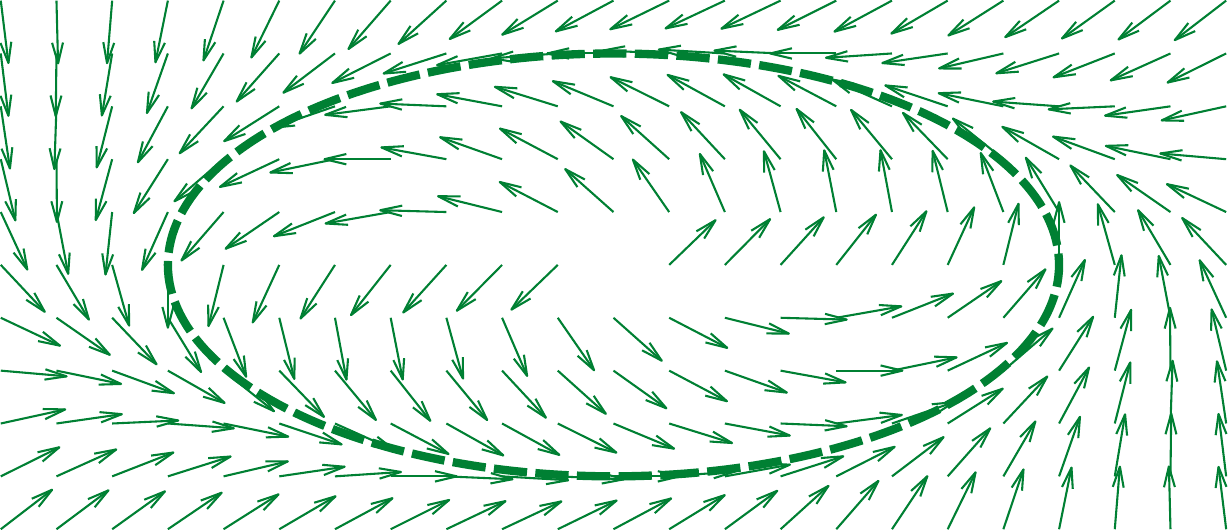}}\\
	\label{Xc_VF_c}
	\subfloat[VF $\sqcup_{\mathcal{Q}_i}(\xi) \hat{\mathcal{X}}_\mathcal{P}(\xi)$]{\includegraphics[width=0.485\linewidth]{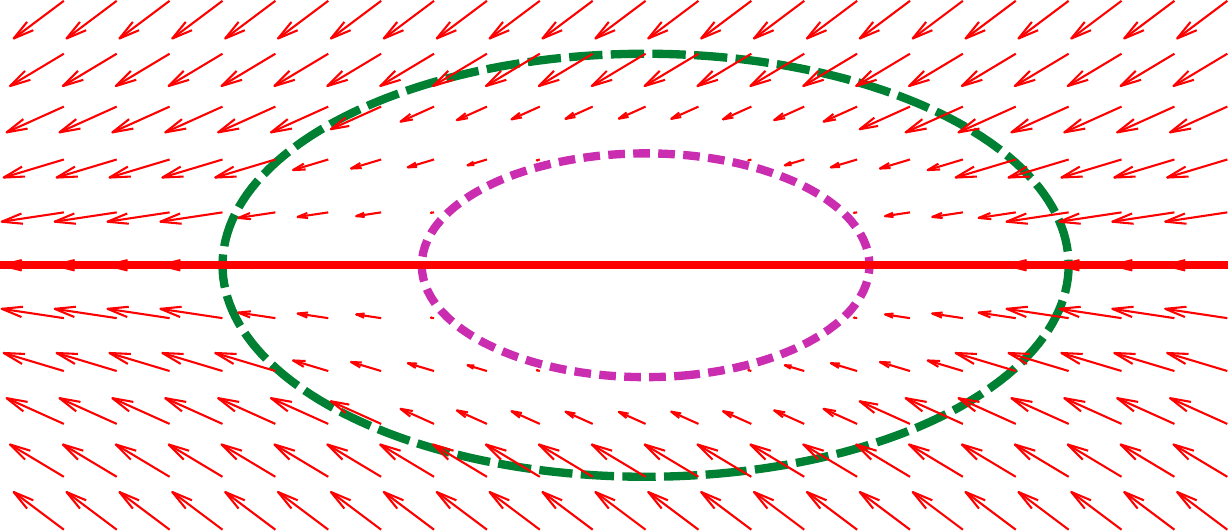}}
	\label{Xc_VF_d}
	\hspace{0.4mm}
	\subfloat[VF $\sqcap_{\mathcal{R}_i}(\xi) \hat{\mathcal{X}}_{\mathcal{R}_i}(\xi)$]{\includegraphics[width=0.485\linewidth]{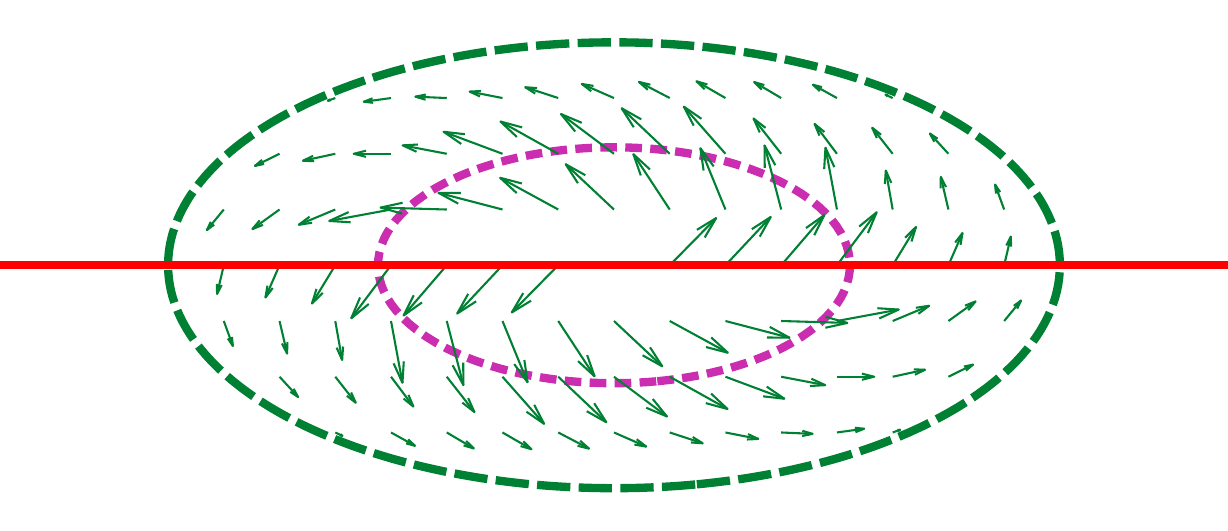}}\\
	\label{Xc_VF_e}
	\subfloat[composite VF $\sqcup_{\mathcal{Q}_i}(\xi) \hat{\mathcal{X}}_\mathcal{P}(\xi) + \sqcap_{\mathcal{R}_i}(\xi) \hat{\mathcal{X}}_{\mathcal{R}_i}(\xi)$]{\includegraphics[width=0.66\linewidth]{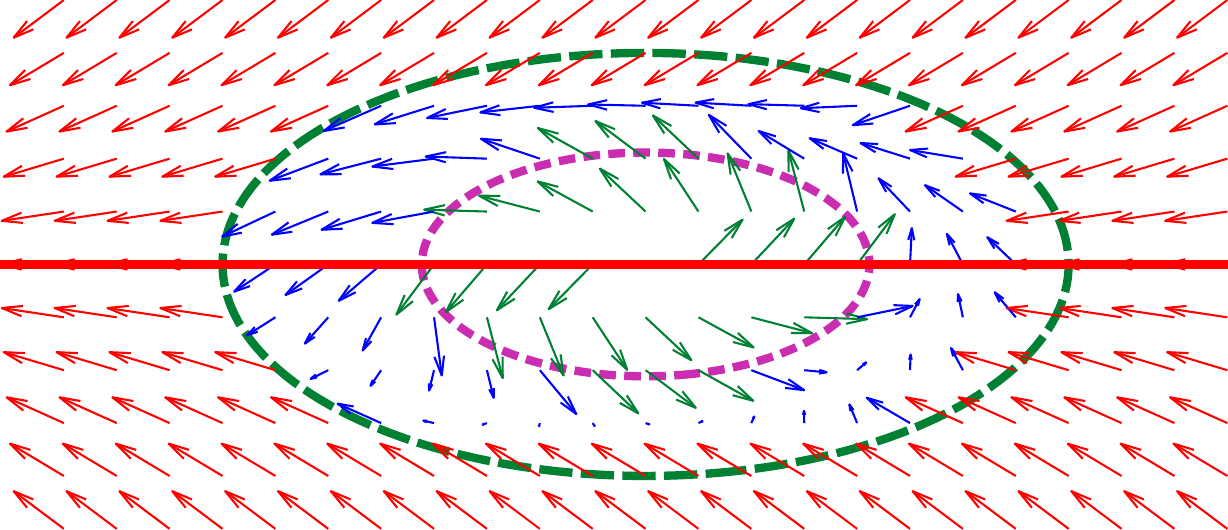}}
	\caption{Illustration of construction of the composite vector field.} 
	\label{Xc_VF}
\end{figure}

The illustration of the construction of the composite vector field is shown in Fig. \ref{Xc_VF}.
The red solid line represents the desired path $\mathcal{P}$, the green dotted line represents the reactive boundary $\mathcal{R}_i$, and the purple dotted line represents the repulsive boundary $\mathcal{Q}_i$.
Each arrow in the subfigures represents a vector of the corresponding vector field at the position.
In Fig.~\ref{Xc_VF}(e), the green arrows belong to the vector field $\hat{\mathcal{X}}_{\mathcal{R}_i}$ generated by the reactive boundary, and are all in $\overline{^\text{in} \mathcal{Q}_i}$.
The red arrows belong to the vector field $\hat{\mathcal{X}}_{\mathcal{P}}$ generated by the desired path, and are all in $\overline{^\text{ex} \mathcal{R}_i}$.
The blue arrows belong to the mixed vector field, the sum of $\hat{\mathcal{X}}_{\mathcal{P}}$ and $\hat{\mathcal{X}}_{\mathcal{R}_i}$, and are all in the mixed area ${^\text{ex} \mathcal{Q}_i} \cap { ^\text{in}} \mathcal{R}_i$. 
More discussion can be found in our previous work \cite{yao2022guiding}.

From this, the \emph{human intention} $\boldsymbol{v}^h_i: \mathbb{R}^2 \to \mathbb{R}^2$ for the $i$-th robot is defined to be the following ``zero-in" guiding vector field:
\begin{equation}
\boldsymbol{v}^h_i(\xi) =  \left(\prod_{j \in \mathcal{I}} \sqcup_{\mathcal{Q}_j}(\xi) \right) \hat{\mathcal{X}}_\mathcal{P}(\xi).
\label{human_intention}
\end{equation}
Notably, it is very difficult for human operators to establish stable communication with each robot because of wireless communication bandwidth limitations and environmental interference.
Therefore, in this paper, we assume that the human operator influences only one of the robots in the formation, and the human intention is propagated through the communication links among the robots and thus indirectly influences the entire MRS.
	The human operator can intuitively draw a desired path via the mouse, the touch screen, or the eye-tracking device, etc., and can choose which robot is affected.

\section{methodology: multi-robot local control}
Without human intervention, the robot should autonomously perform tasks through its own perception and decision-making in firefighting scenarios\footnote{In fact, our human-robot shared control approach can be used in a wide range of studies not limited to firefighting missions, and this paper uses firefighting missions as an illustrative example.}. For such tasks, we assume that 1) a single robot has limited firefighting capability, and a single fire source requires multiple robots to work simultaneously to extinguish it. 2) individual robots should not leave the formation for a long period of time or long distances to maintain communication with others. In this section, we design a local controller for multi-robot coordination, including two subtasks: target selection (robot intention) and formation control, and these subtasks are accomplished by using different vector fields that will be designed in the following subsections.
For convenience, we use the notation $\mathbb{Z}_{a}^{b}$ to denote the set of integers $\{a,a+1,\dots,b\}$ for $a<b$.

\subsection{Robot Intention}
Finding and prioritizing the most threatening fire source to extinguish is the primary task for the MRS. Each robot prioritizes the fire to be extinguished based on its perception of the fire situation, and generates the \emph{robot intention} $\boldsymbol{v}_t^i$, which points to the largest fire source observed by the $i$-th robot.
\begin{figure}[h!]
	\includegraphics[width=\linewidth]{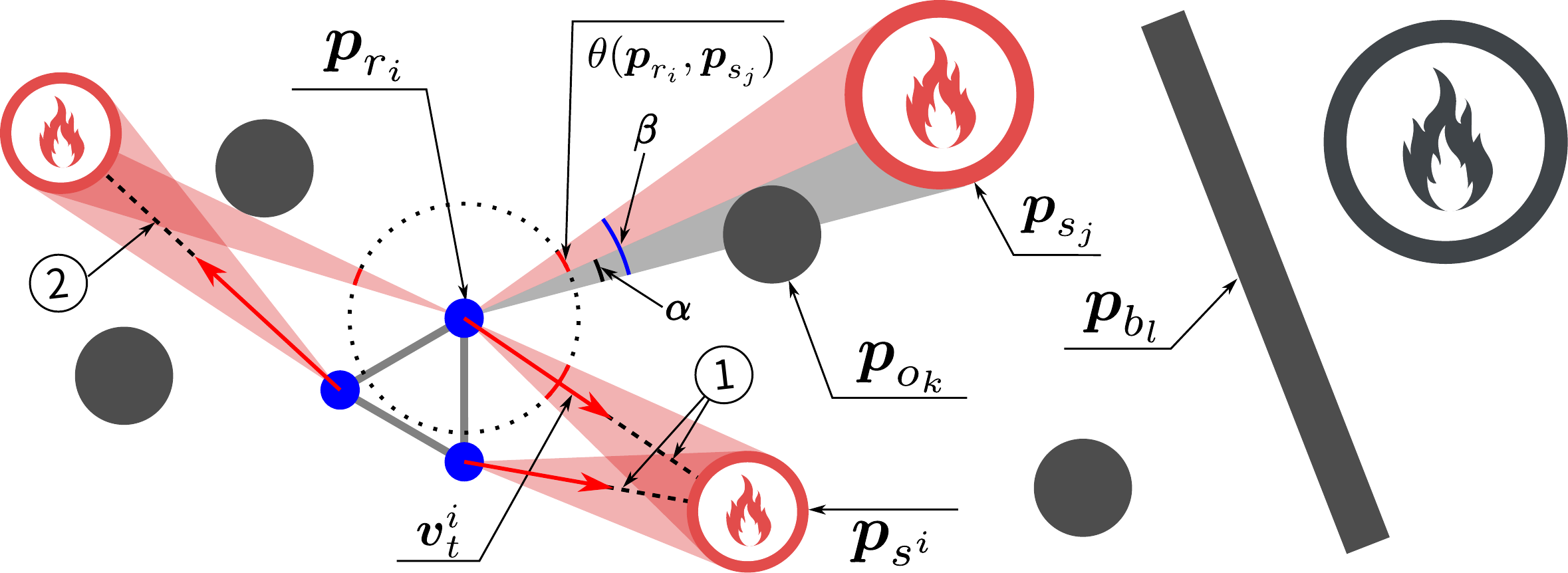}
	\caption{The environment influences the observed fire values.} 
	\label{fire_observation}
\end{figure}

Suppose there are $N$ homogeneous omnidirectional robots, with positions denoted by $\boldsymbol{p}_{r_i} \in \mathbb{R}^2, i\in\mathbb{Z}_1^N$, $M$  fire sources, with positions denoted by $\boldsymbol{p}_{s_i} \in \mathbb{R}^2, i\in\mathbb{Z}_1^M$, $P$ circle obstacles, with positions denoted by $\boldsymbol{p}_{o_i} \in \mathbb{R}^2, i\in\mathbb{Z}_1^P$, and $Q$ barriers, with positions denoted by $\boldsymbol{p}_{b_i} \in \mathbb{R}^2, i\in\mathbb{Z}_1^Q$, in the two-dimensional space.
In this paper, the observed severity of the $j$-th fire source by the $i$-th robot is described by the \emph{observed fire value} $\theta(\boldsymbol{p}_{r_i},\boldsymbol{p}_{s_j}) \ge 0$. This value increases if the fire becomes more severe and decreases if obstacles obstruct the robot's observation. It is worth emphasizing that, in practice, the observed fire value is \emph{directly measured by the onboard sensors}. However, for simplicity, we use a circle centered at the fire source to represent the severity of the fire; the more severe the fire, the larger the radius. 
Then the observed fire value is approximated as the angle calculated by subtracting the angle of view obscured by the obstacle ($\alpha$ in Fig. \ref{fire_observation}) from the angle of view of the fire ($\beta$ in Fig. \ref{fire_observation}) when there are no obstacles. Note that if $\alpha>\beta$, i.e., the fire source is totally obstructed by obstacles, then the observed fire value is zero. Fig. \ref{fire_observation} graphically illustrates how the fire value is determined in an environment with fire sources, robots and obstacles.
Therefore, the $i$-th robot can select the most severe fire source as its prioritized target $s^i$, which is defined as:
\begin{equation}
s^i= \mathop{\arg \max}\limits_{j \in \mathbb{Z}_1^M} \theta(\boldsymbol{p}_{r_i},\boldsymbol{p}_{s_j}).
\end{equation}
%
In some cases, robots may not choose the same target. For example, in Fig.  \ref{fire_observation}, two robots choose \ding{192} as their prioritized target while one robot has a different target: \ding{193}.  
Then, similar to method \eqref{vector field_phi}, we define a vector field for the $i$-th robot based on the target $s^i$:
\begin{equation}
\mathcal{X}^t_{i}(\xi) = -\theta(\boldsymbol{p}_{r_i},\boldsymbol{p}_{s^i})\phi_{s^i}(\xi)\nabla \phi_{s^i}(\xi),
\label{target}
\end{equation}
where $\phi_{s^i}(\xi)=\vert \xi-\boldsymbol{p}_{s^i} \vert ^2-{R_{s^i}}^2$, and $\boldsymbol{p}_{s^i}$, $R_{s^i}$ represent the position and radius of the prioritized target fire source $s^i$, respectively.
Note that the tangential component is removed and only the  orthogonal component remains. Therefore, following this vector field \eqref{target}, the robot would move directly towards the fire target. 


\subsection{Formation Control}
As some robots have different intentions, it may cause them to be separated from the team for too long or too far away, thus breaking away from the formation and resulting in a failure to work together to accomplish the mission.
To address this problem, we adopt the distance-based formation control method to maintain the formation of MRS while moving.

Then, using \eqref{vector field_phi}, we can define a composite vector field for the $i$-th robot based on the neighbors of the $i$-th robot:
\begin{equation}
\mathcal{X}_i^f(\xi) = \sum _{j \in \mathcal{N}_i}- \left\vert \Vert\xi-\boldsymbol{p}_{r_j} \Vert-d_{ij}  \right\vert \nabla \hat{\phi}_{r_j}(\xi),
\end{equation}
where $d_{ij}$ represents the desired distance between the $i$-th and $j$-th robot, $\phi_{r_j}(\xi)=\Vert \xi - \boldsymbol{p}_{r_j}\Vert ^2 - d_{ij}^2$, $\hat{(\cdot)}$ is the normalization notation (i.e., it normalizes the gradient $\nabla \phi_{r_j}$), 
and $\mathcal{N}_i$ represents the neighbor set of the $i$-th robot\footnote{If a robot has a (bidirectional) communication link with the $i$-th robot, then it is a neighbor of the $i$-th robot.
	For example, when a communication range $R_c$ is set, other robots within the range are neighbors of $i$-th robot. Human operators can also initialize the communication connection topology among robots before the task execution.}. Due to the constraints of formation control, robots with greater robot intention will dominate when each robot moves to its selected target. Therefore, the target selection of the entire MRS will eventually become consistent.

%

Then, similar to \eqref{human_intention}, the VF induced by the target (i.e., the robot intention), and the VF for formation control are defined respectively as follows: 
\begin{equation}
\boldsymbol{v}_i^t(\xi) = \left(\prod_{j \in \mathbb{Z}_1^{M+P+Q}}  \sqcup_{\mathcal{Q}_j}(\xi)\right) \mathcal{X}_i^t(\xi),
\label{robot_intention}
\end{equation}
\begin{equation}
\boldsymbol{v}_i^f(\xi) = \left(\prod_{j \in \mathbb{Z}_1^{M+P+Q}}  \sqcup_{\mathcal{Q}_j}(\xi)\right)\mathcal{X}_i^f(\xi).
\label{formation_componet}
\end{equation}

\section{methodology: shared control for the human-multi-robot}

So far, we have obtained the VFs based on the local controller and HI-GVF. Then, using a shared control method, the human intention and the robot intention can be reasonably fused to generate more appropriate actions, which can improve the execution efficiency of the system. In this section, a layered shared control framework is proposed to propagate the human intention and blend the VFs, effectively influencing the entire MRS by controlling one robot.

\subsection{Upper Layer: Intention Field Model}
In the upper layer of the framework, we use an intention field model to allow the human operator to control one robot and to propagate the intention among the robots, thus affecting the whole system.
As stated in the preceding sections, $\boldsymbol{v}^t_i \in \mathbb{R}^2$ obtained by \eqref{robot_intention} represents the \emph{robot intention}, and $\boldsymbol{v}^h_i \in \mathbb{R}^2$ obtained by \eqref{human_intention} represents the \emph{human intention} received by the human-influenced robot. 
In each control cycle, the $i$-th robot updates the \emph{shared intention} $\boldsymbol{v}_i^s$ at time step $k+1$ as follows\footnote{For simplicity, the notation of the shared intention $\boldsymbol{v}_i^s(\xi(k+1))$ at time step $k+1$ is simplified as $\boldsymbol{v}_i^s(k+1)$; the same simplification is used for other notation. }:
\begin{equation}
\begin{aligned}
\boldsymbol{v}_i^s(k+1)= & -\omega_0 \boldsymbol{v}_i^s(k) +\omega_1\left(\boldsymbol{v}_i^t(k)-\boldsymbol{v}_i^s(k)\right) \\
& +\omega_2 \sum_{j \in \mathcal{N}_i} \Phi\left(\boldsymbol{v}_j^s(k)-\boldsymbol{v}_i^s(k)\right) \\
& +\omega_3 \Psi_i\left(\boldsymbol{v}_i^h(k)-\boldsymbol{v}_i^s(k)\right),
\end{aligned}
\label{shared_intention1}
\end{equation}
where $\boldsymbol{v}_i^s(0)=\boldsymbol{0}$,  $\omega_0,\omega_1,\omega_2,\omega_3 \in \mathbb{R}_+$ are weighted coefficients that regulate the degree of influence of the robot intention, the neighbor intention, and the human intention in the update process, and $\Psi_i$ is an indicator function that $\Psi_i = 1$ holds when the $i$-th robot has received the human intention $\boldsymbol{v}_h^i$ and $\Psi_i= 0$ otherwise, and $\Phi(\cdot)$ is a dead zone function defined as:
\begin{equation}
\Phi(\boldsymbol{x})=\left\{\begin{array}{lll}
(\Vert\boldsymbol{x}\Vert-\epsilon) \boldsymbol{x} /\Vert \boldsymbol{x}\Vert&, &\Vert\boldsymbol{x}\Vert>\epsilon \\
\mathbf{0}&, &\Vert\boldsymbol{x}\Vert\leq \epsilon
\end{array},\right.
\end{equation}
where $\epsilon \geq 0$.
The intention field model can help the human operator to control effectively by abstracting the human input as an intention field value and fusing it with the robot intention.
	The shared intention can be viewed as a virtual control intention generated by robot intentions and human intentions.
	Assuming that the desired robot motion speed is $\mathcal{C} >0$, the final shared intention can be represented as
	
	\begin{equation}
	\hat{\boldsymbol{v}}_i^s=\mathcal{C}\boldsymbol{v}_i^s/\|\boldsymbol{v}_i^s\|.
	\label{final_s}
	\end{equation}
The model has the following properties:
\begin{itemize}
	\item [1)]
	\emph{Local similarity}. If the $i$-th robot is close to the $j$-th robot, then $\boldsymbol{v}_s^i$ and $\boldsymbol{v}_s^j$ have similar values.
	\item [2)]
	\emph{Space separation}. If the $i$-th robot is far from the $j$-th robot, then variation in $\boldsymbol{v}_s^i$ has little impact on  $\boldsymbol{v}_s^j$.
\end{itemize}    

Local similarity allows nearby robots to be controlled simultaneously by a human operator, while space separation enables different parts of the MRS to be controlled separately. 
With the intention field model, each robot updates itself under the influence of its neighbors and the human operator.
When one robot receives the human intention, the inter-robot network implicitly propagates that intention to each robot.
This result is fundamentally different from the ``leader-follower'' model, in which followers just react passively to the actions of their neighbors.
The stability analysis of the intention field can be referred to in the Appendix-A.

\subsection{Lower Layer: Policy-Blending Model}

	In the lower layer of the framework, a consensus network is used to blend the shared intention generated by the upper layer and VF of formation control generated by the local controller. 
	We can obtain the final velocity $\boldsymbol{v}^i$ as follows:
	\begin{equation}
	\boldsymbol{v}_i=\lambda_i \hat{\boldsymbol{v}}_i^s+(1-\lambda_i) \boldsymbol{v}_i^f,
	\label{fusing_velocity}
	\end{equation}
	where $\lambda_i=\lambda(\|\boldsymbol{v}_i^s\|,\|\boldsymbol{v}_i^f\|)$, $\boldsymbol{v}_i^s$, $\boldsymbol{v}_i^f$ and $\hat{\boldsymbol{v}}_i^s$ are obtained by \eqref{shared_intention1}, \eqref{formation_componet} and \eqref{final_s}, respectively.
	
	The weighting function $\lambda: \mathbb{R}^+ \times \mathbb{R}^+ \to [0,1]$ need to satisfy the following conditions:
	\begin{itemize}
		\item[1)]
		$\lambda$ is continuous, non-negative, and $\lambda(0,0)=0$.
		\item[2)]
		Let $f_b(a)=\lambda(a,b)$, then $f_b(a)$ is strictly increasing.
		\item[3)]
		Let $f_a(b)=\lambda(a,b)$, then $f_a(b)$ is strictly decreasing.
		\item[4)]
		$\lim_{a \to \infty} f_b(a)=1$ and $\lim_{b \to \infty} f_a(b)=0$
	\end{itemize}
	We design a function that satisfies these conditions
	\begin{equation}
	\lambda(\|\boldsymbol{v}_i^s\|,\|\boldsymbol{v}_i^f\|)=\frac{k_s\|\boldsymbol{v}_i^s\|}{k_s\|\boldsymbol{v}_i^s\|+k_f\|\boldsymbol{v}_i^f\|}
	\label{lambda}
	\end{equation} 
	where $k_s,k_f \in R^+$ are constants.

	The final velocity $\boldsymbol{v}_i$ integrates all the VFs so that with suitable policy-blending, the MRS can maintain its formation and avoid obstacles while reaching the target.
	The stability analysis of the consensus network can be referred to in the Appendix-B.

	\section{collision avoidance based on safety barrier certificates}
	An analysis of the stability of the whole formation when human and target intentions are used as inputs has already been given in the previous section.
	However, when the formation performs obstacle avoidance, it is difficult to ensure the stability of the whole formation anymore. 
	Therefore, it is very important and necessary to adopt suitable collision avoidance methods.
	We use a safety barrier certificates method \cite{wang2017safety} based on optimization in a minimally invasive way, which ensures stability in obstacle avoidance for formation.
	\subsection{collision avoidance among robots}
	Although formation control can keep robots at a distance from each other to some extent, it does not guarantee that robots will not collide with each other at all.
	Therefore, we still need to consider collision avoidance within the formation.
	By first defining a set of safe states, we use the control barrier function (CBF) to formally guarantee forward invariance of the desired set, i.e., if the system starts in the safe set, it stays in the safe set.
	A CBF is similar to a Control Lyapunov Function in that they ensure certain properties of the system without the need to explicitly compute the forward reachable set.
	
	For $i,j\in \mathbb{Z}_1^N,i\neq j$, we define a function
	\begin{equation}
	\mathfrak{h}_{i j}(\boldsymbol{p}_{r_i}, \boldsymbol{p}_{r_j})=\|\boldsymbol{p}_{r_i}-\boldsymbol{p}_{r_j}\|^2-R_r^2,
	\end{equation}
	where $R_r >0$ is the safe distance among robots.
	The function $\mathfrak{h}_{i j}$ reflects whether two robots maintain a safe distance.
	Therefore, the pairwise safe set can be defined as
	\begin{equation}
	\mathfrak{S}_{ij}=\{(\boldsymbol{p}_{r_i}, \boldsymbol{p}_{r_j}) \vert \mathfrak{h}_{i j}(\boldsymbol{p}_{r_i}, \boldsymbol{p}_{r_j} ) \geq 0 \quad \forall i \neq j\}.
	\end{equation}
	Using $\mathfrak{h}_{i j}$, we can define a barrier function
	\begin{equation}
	B_{ij}(\boldsymbol{p}_{r_i},\boldsymbol{p}_{r_j})=\frac{1}{\mathfrak{h}_{i j}(\boldsymbol{p}_{r_i},\boldsymbol{p}_{r_j})}.
	\label{barrier_function}
	\end{equation}
	
	We use the optimization method to adjust the final velocity $\boldsymbol{v}_i$ so that the new velocity $\hat{\boldsymbol{v}}_i$ enables collision-avoidance.
	Specifically, the new velocity $\hat{\boldsymbol{v}}_i$ should subject to the constraint $\dot{B}_{ij} \leq \alpha/B_{ij}$, where $\alpha>0$ is a constant.
	Using \eqref{barrier_function}, we can obtain $\dot{B}_{ij}=-\dot{\mathfrak{h}}_{ij}/\mathfrak{h}_{ij}^2=-2(\boldsymbol{p}_{r_i}-\boldsymbol{p}_{r_j})^\top(\hat{\boldsymbol{v}}_i-\hat{\boldsymbol{v}}_j)/\mathfrak{h}_{i j}^2$, so the constraint $\dot{B}_{ij} \leq \alpha/B_{ij}$ can be rewritten as 
	\begin{equation}
	-(\boldsymbol{p}_{r_i}-\boldsymbol{p}_{r_j})^\top \hat{\boldsymbol{v}}_i+(\boldsymbol{p}_{r_i}-\boldsymbol{p}_{r_j})^\top \hat{\boldsymbol{v}}_j\leq\frac{\alpha}{2}\mathfrak{h}_{i j}^3.
	\label{constraint}
	\end{equation}
	In order to ensure that the main control objectives are met to the greatest extent possible, the collision avoidance strategy should modify the target velocity as little as possible.
	Therefore, we calculate the new velocity $\hat{\boldsymbol{v}}=(\hat{\boldsymbol{v}}_1,\cdots,\hat{\boldsymbol{v}}_N) \in \mathbb{R}^{2N}$ by the quadratic program (QP) as follows:
	\begin{equation}
	\begin{aligned}
	\hat{\boldsymbol{v}}^*=\underset{\hat{\boldsymbol{v}} \in \mathbb{R}^{2N}}{\text{argmin}} & \sum_{i=1}^N\left\|\hat{\boldsymbol{v}}_i-\boldsymbol{v}_i\right\|^2 \\
	\text { s.t. } & \eqref{constraint}, \quad \forall i \neq j 
	\end{aligned}
	\end{equation}
	This optimization problem is centralized as it requires the positions and desired velocities of all robots.
	Further, we can improve it to a distributed form for each robot $i \in \mathbb{Z}_1^N$:
	\begin{equation}
	\begin{aligned}
	\hat{\boldsymbol{v}}_i^*=\underset{\hat{\boldsymbol{v}}_i \in \mathbb{R}^{2}}{\text{argmin}} & \left\|\hat{\boldsymbol{v}}_i-\boldsymbol{v}_i\right\|^2 \\
	\text { s.t. } & 
	-(\boldsymbol{p}_{r_i}-\boldsymbol{p}_{r_j})^\top \hat{\boldsymbol{v}}_i\leq\frac{\alpha}{4}\mathfrak{h}_{i j}^3, \quad \forall j \in \mathcal{N}_i
	\end{aligned}
	\end{equation}
	Therefore, robot $i$ only considers the positions of its neighbors, and does not need to get the desired velocity of other robots.
	
	\subsection{Collision avoidance between robots and obstacles}
	Note that the safety barrier certificates can also deal with static or moving obstacles by treating obstacles as agents with no inputs.
	Similar to the method of collision avoidance among robots, the robot needs to maintain a safe distance between obstacles.
	Therefore, a function can be designed to ensure that the robot $i$ does not collide with the obstacle $\boldsymbol{o}_k$
	\begin{equation}
	\bar{\mathfrak{h}}_{i k}(\boldsymbol{p}_{r_i}, \boldsymbol{p}_{o_k})=\|\boldsymbol{p}_{r_i}-\boldsymbol{p}_{o_k}\|^2-R_o^2
	\end{equation}
	where $R_o$ is the safe distance between the robot and the obstacle.
	
	Similar to \eqref{constraint}, we can obtain the constraint of robot avoiding obstacles:
	\begin{equation}
	-(\boldsymbol{p}_{r_i}-\boldsymbol{p}_{o_k})^\top \hat{\boldsymbol{v}}_i\leq\frac{\beta}{2}\bar{\mathfrak{h}}_{i j}^3,
	\label{constraint_bos}
	\end{equation}
	where $\beta >0$ is a constant.
	Then, the new velocity $\hat{\boldsymbol{v}}_i$ can be calculated by the QP as follows:
	\begin{equation}
	\begin{aligned}
	\hat{\boldsymbol{v}}_i^*=\underset{\hat{\boldsymbol{v}}_i \in \mathbb{R}^{2}}{\text{argmin}} & \left\|\hat{\boldsymbol{v}}_i-\boldsymbol{v}_i\right\|^2 \\
	\text { s.t. } & 
	-(\boldsymbol{p}_{r_i}-\boldsymbol{p}_{r_j})^\top \hat{\boldsymbol{v}}_i\leq\frac{\alpha}{4}\mathfrak{h}_{i j}^3, \quad \forall j \in \mathcal{N}_i\\
	& -(\boldsymbol{p}_{r_i}-\boldsymbol{p}_{o_k})^\top \hat{\boldsymbol{v}}_i\leq\frac{\beta}{2}\bar{\mathfrak{h}}_{i k}^3.
	\end{aligned}
	\end{equation}
	
	It is worth noting that the collision avoidance behavior may affect the convergence of the robot formation, and even deadlocks may occur at certain positions (the deadlock problem is still an intractable challenge at this point), resulting in robots not being able to move. 
	Thus, to ensure that multi-robot cooperative tasks and collision avoidance behavior are feasible, we assume that the required parameters have been properly designed, that the initial Euclidean distances between the robots are all greater than the safe distance $R_r$, and that the initial distances between robots and obstacles are all greater than the safe distance $R_o$. When the robots are unable to move due to the deadlock problem, we temporarily disengage the robots from the deadlock position by increasing the human intent.

\section{Validation}
To verify the effectiveness of our proposed method, we designed a series of experiments based on simulation and physical MRS platforms, respectively.

	Given the function $\varphi_i$ and $c_i$ as in \eqref{Q-smooth}.
	We design the function $S_i$ and $Z_i$ as follows:
	\begin{equation}
	S_i(\xi)=\frac{f_1(\xi)}{f_1(\xi)+f_2(\xi)}
	\end{equation}
	\begin{equation}
	Z_i(\xi)=\frac{f_2(\xi)}{f_1(\xi)+f_2(\xi)}
	\end{equation}
	where $f_1$ and $f_2$ are defined as follows:
	\begin{equation}
	\begin{array}{c}
	f_1(\xi)=\left\{\begin{array}{ll}
	0 & \xi \in\{\varphi(\xi) \leq c_i\} \\
	\exp \left(\frac{l_1}{c_i-\varphi(\xi)}\right) & \xi \in\{\varphi(\xi)>c_i\}
	\end{array}\right. \\
	f_2(\xi)=\left\{\begin{array}{ll}
	\exp \left(\frac{l_2}{\varphi(\xi)}\right) & \xi \in\{\varphi(\xi)<0\} \\
	0 & \xi \in\{\varphi(\xi) \geq 0\}
	\end{array}\right.
	\end{array}
	\end{equation}
	where $l_1 > 0$, $l_2 > 0$ are used to change the decaying or increasing rate.

\begin{figure}[h!]
	\centering
	\includegraphics[width=\linewidth]{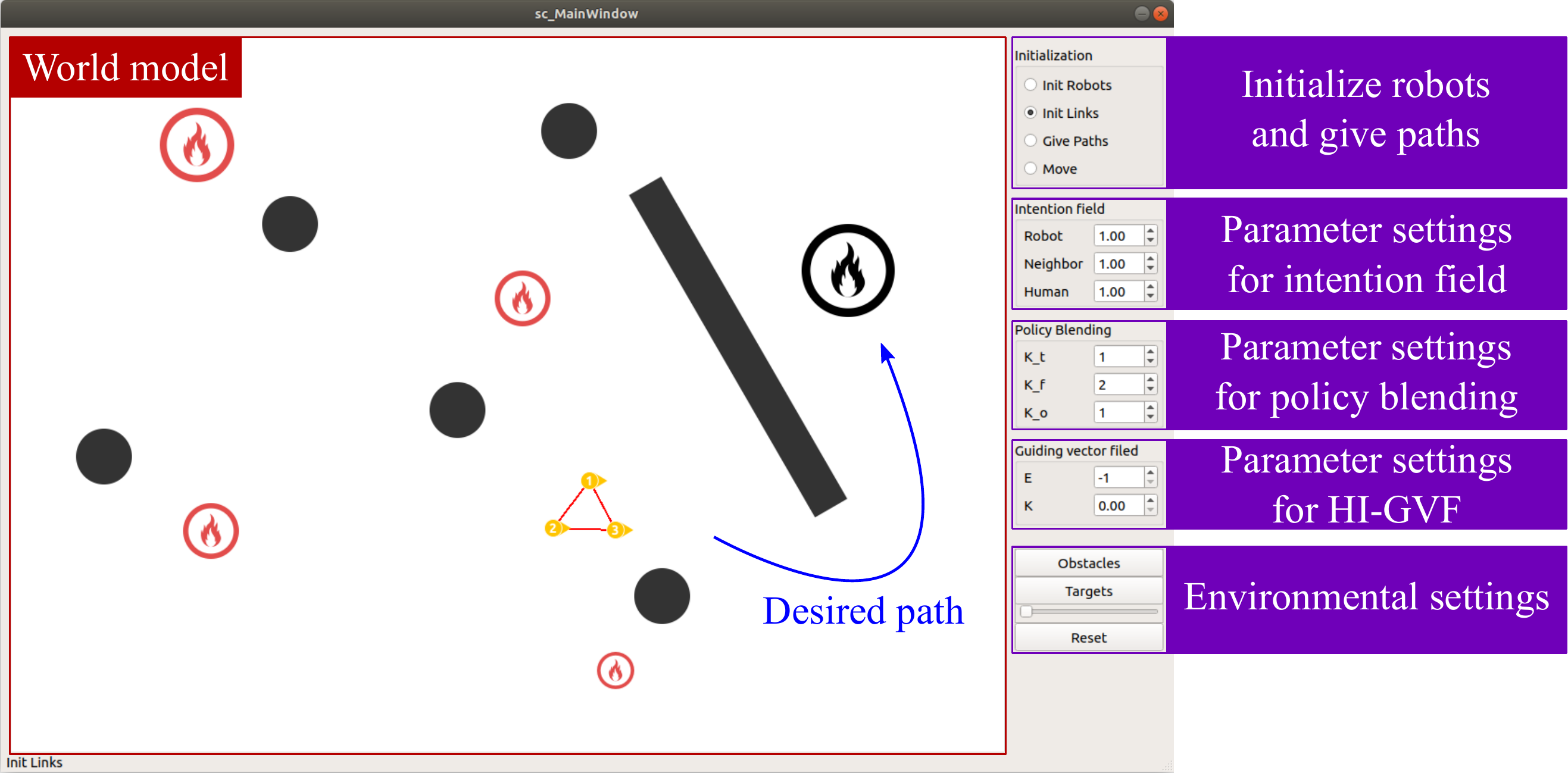}
	\caption{The details of simulation environment.} 
	\label{Qt}
\end{figure}

\begin{figure}[h!]
	\centering
	\subfloat[MRS follows the desired path]{\includegraphics[width=\linewidth]{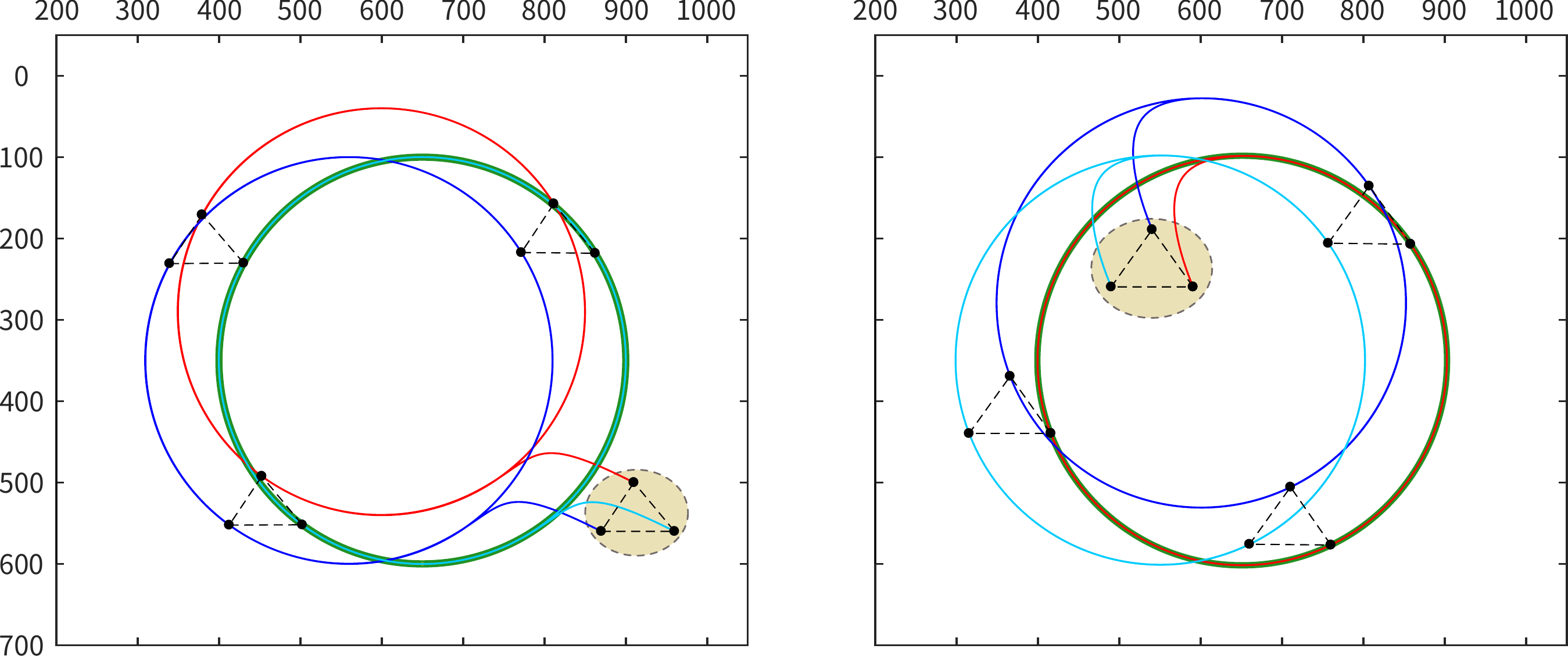}}
	\label{circle1}\\
	\subfloat[MRS follows the desired path with changing affected robots]
	{\includegraphics[width=\linewidth]{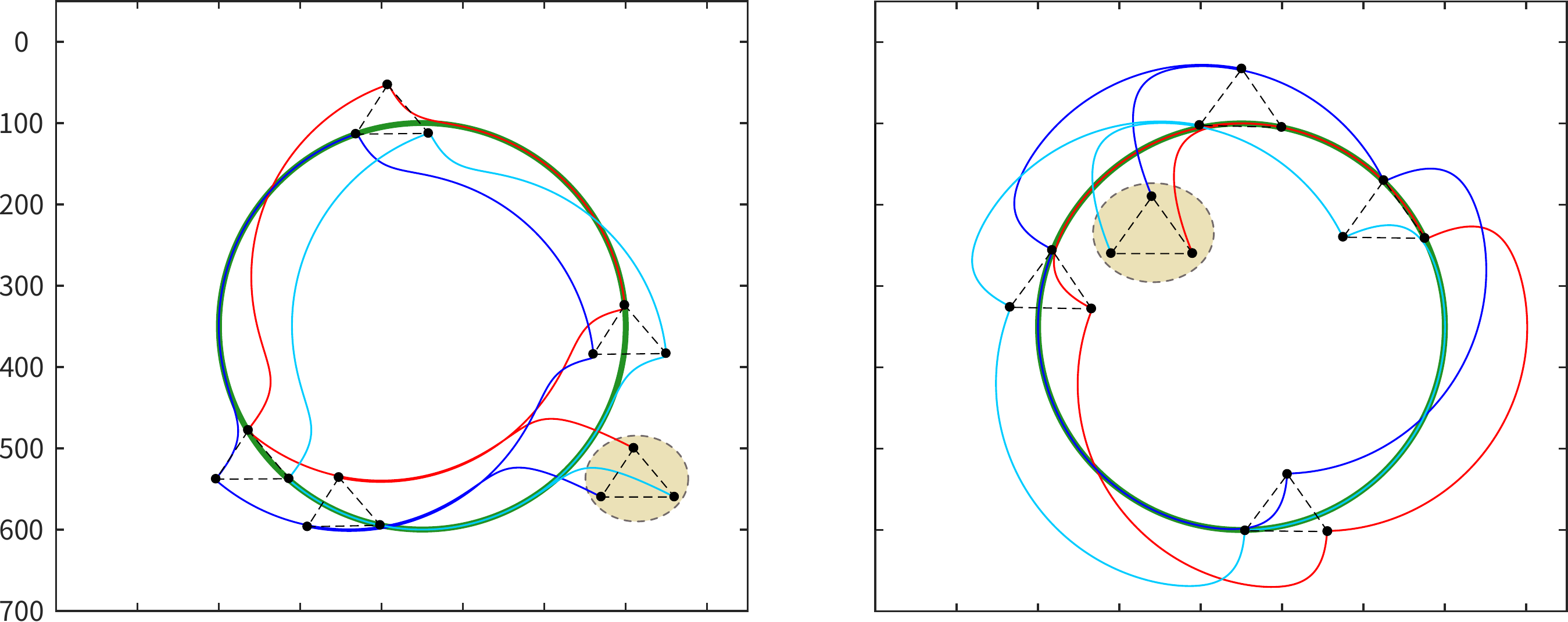}}
	\label{circle2}\\
	\subfloat[MRS follows the desired path when encountering obstacles]{\includegraphics[width=\linewidth]{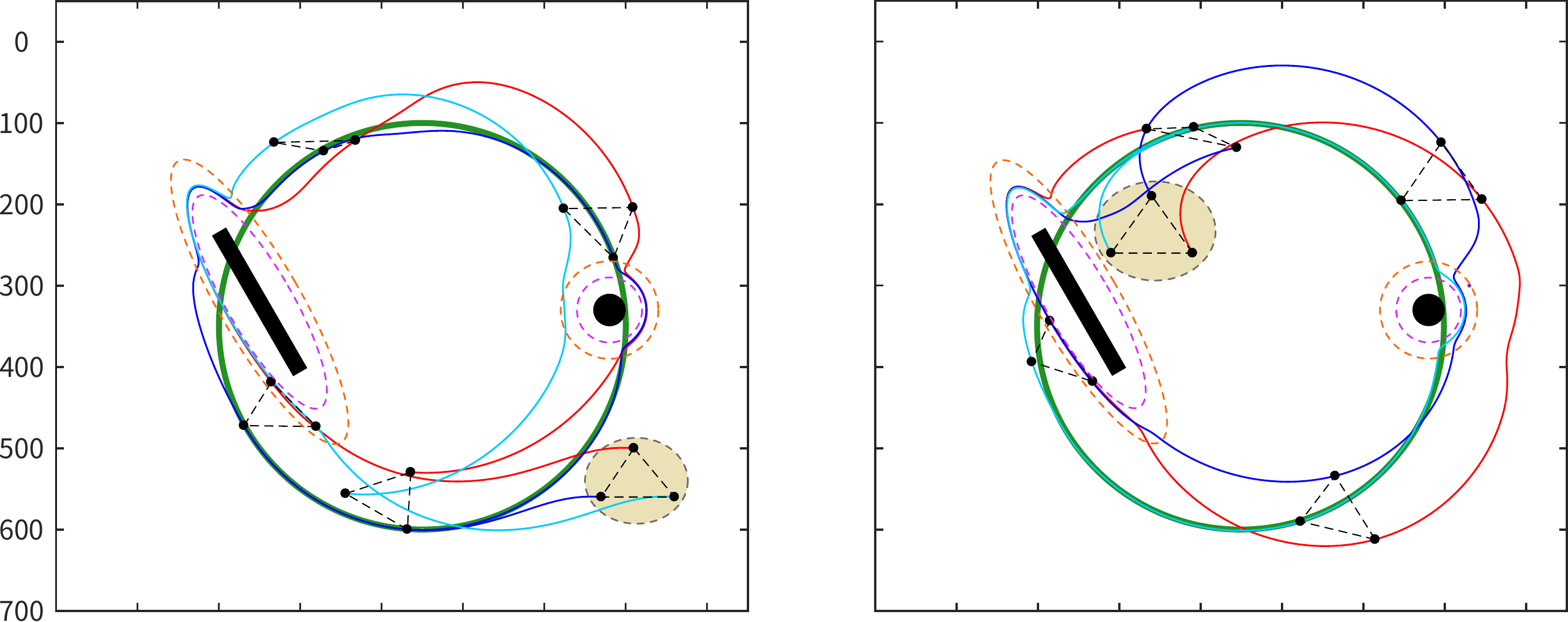}}
	\label{circle3}\\
	\includegraphics[width=\linewidth]{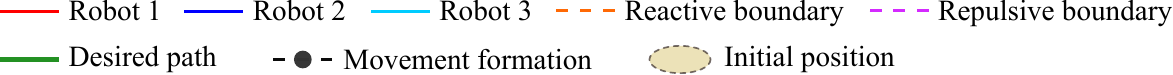}
	\caption{The robot trajectories when MRS follows the desired path.} 
	\label{circle}
\end{figure}

\begin{figure*}[h!]
	\centering
	\subfloat[MRS performs task without human intervention]{\includegraphics[width=0.33\linewidth]{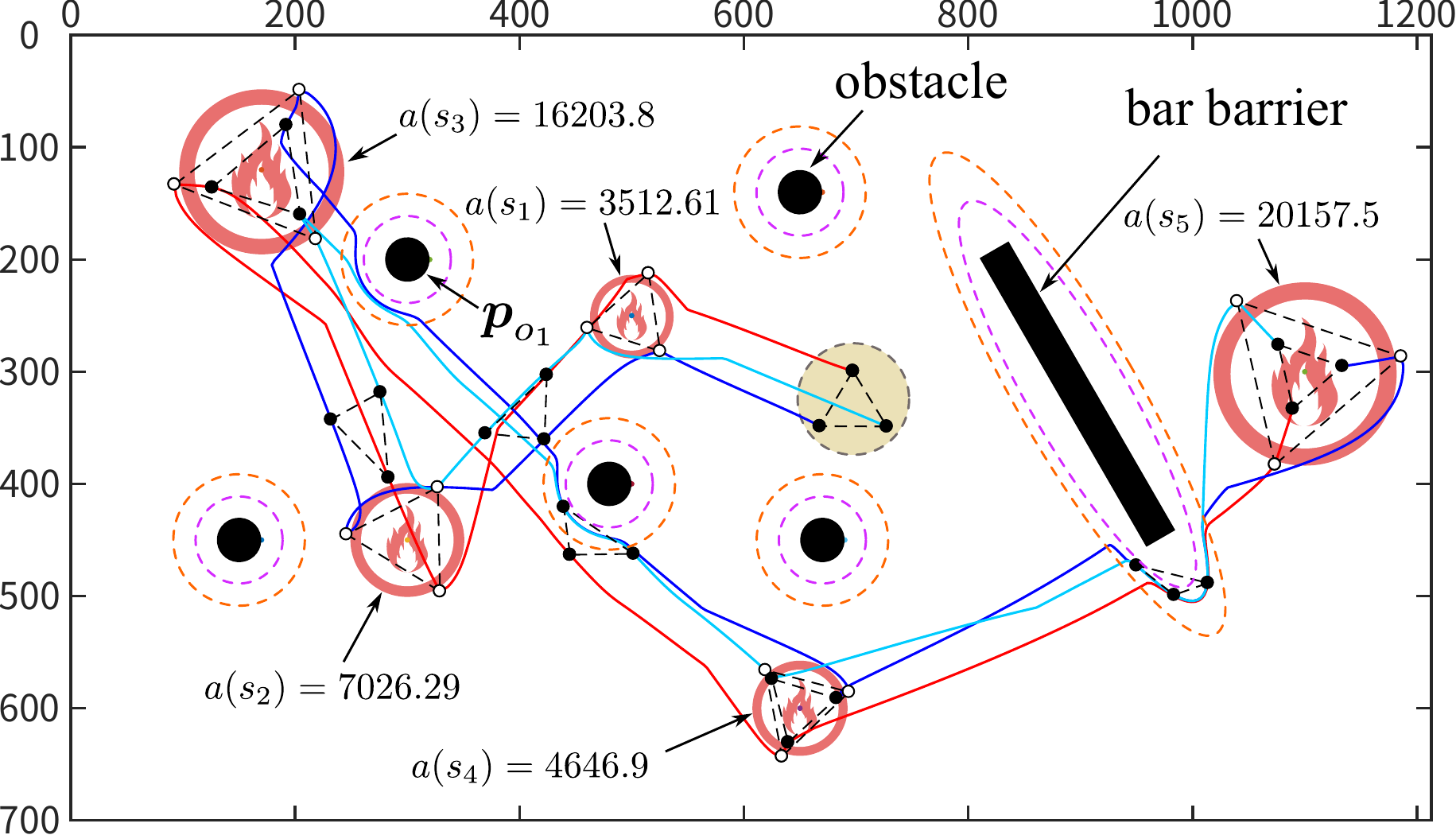}}
	\subfloat[Changing target priority with human]{\includegraphics[width=0.33\linewidth]{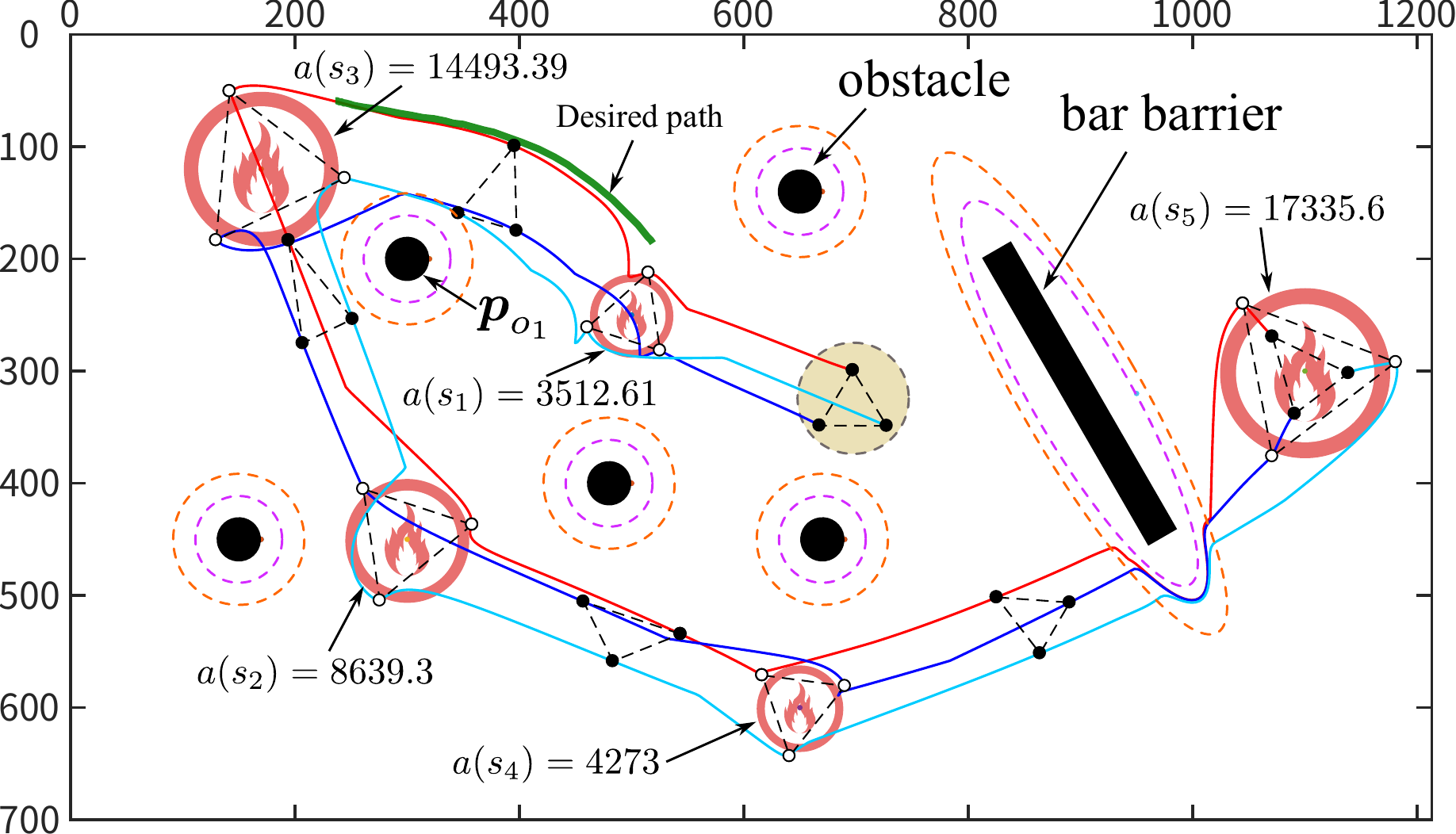}}
	\subfloat[Finding undetected target]{\includegraphics[width=0.33\linewidth]{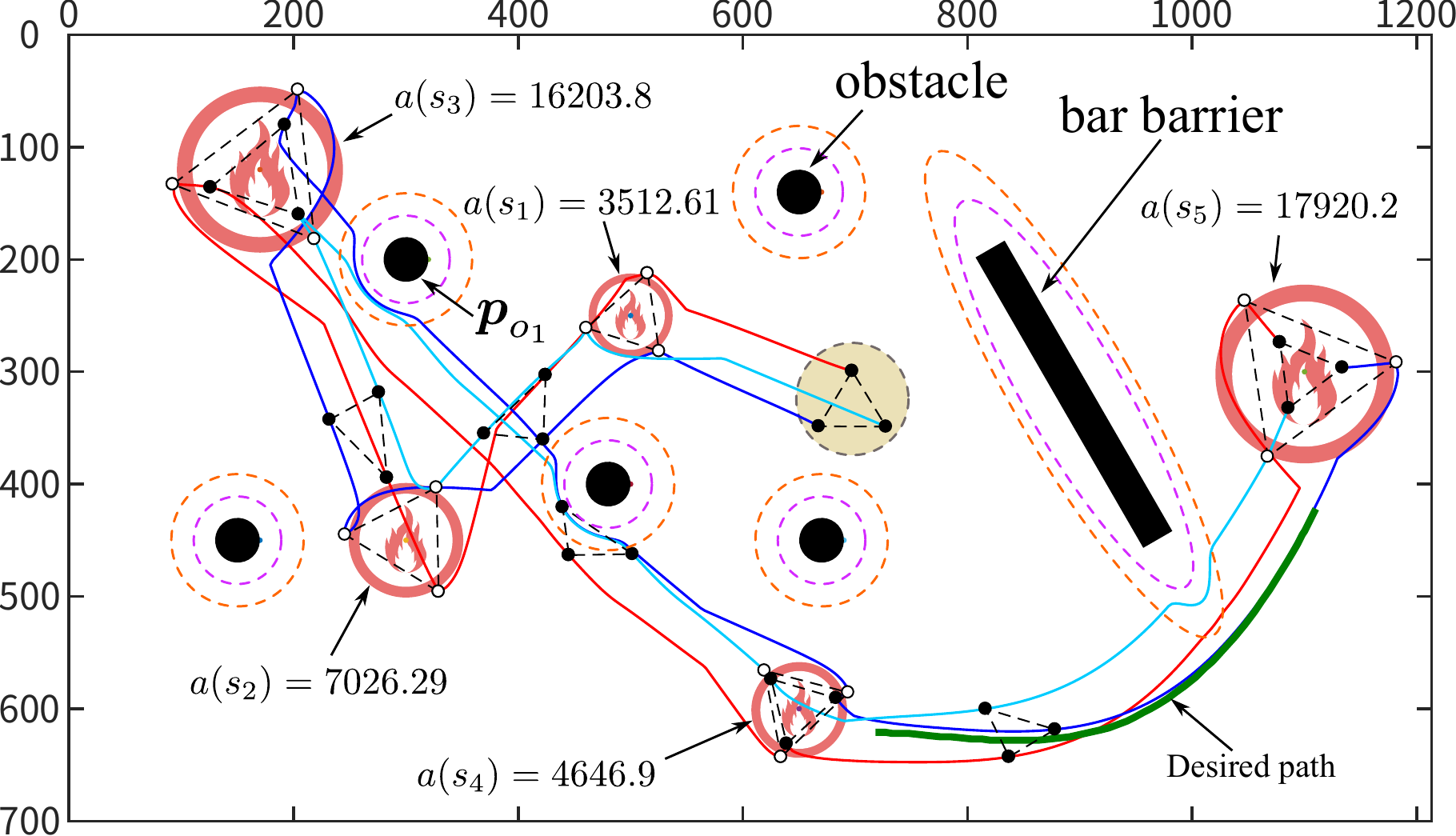}}\\
	\includegraphics[width=\linewidth]{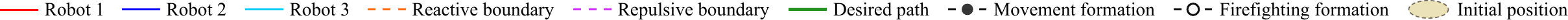}
	\caption{MRS perform task with human intervention.} 
	\label{human_intervention}
\end{figure*}

\subsection{Simulation}

We design a human-computer interface based on QT for simulation experiments as shown in Fig. \ref{Qt}.
The operator can arbitrarily initialize the robot position and connection, modify parameters, etc.
We set up the following three types of experiments to verify the effectiveness of the algorithm.

\subsubsection{MRS follows the desired path without robot intention}
Given the desired path as a circle with center (250,300) and radius 250, the initial positions of the three robots are (590,260), (540,190), (490,260), and (910,500), (870,560), (960,560), which are located inside and outside the circle, respectively.

\textbf{Case 1}: MRS follows the desired path. Selecting a certain robot to be affected, this affected robot is able to follow the desired path from inside or outside the circle.
Besides, the other robots can keep moving in formation, thus generating similar trajectories, as shown in Fig. \ref{circle}(a).

\textbf{Case 2}: MRS follows the desired path with changing affected robots. When the MRS follows the desired path, the human can choose any robot affected. Thanks to the influence of the intention field, the entire MRS can respond quickly and continue to follow the desired path, as shown in Fig. \ref{circle}(b).

\textbf{Case 3}: MRS follows the desired path when encountering obstacles. In this case, there are bar barriers and circle obstacles in the environment. The robots will temporarily deviate from the desired path and change formation to avoid obstacles, which can be irregularly shaped, as shown in Fig. \ref{circle}(c).

\subsubsection{Shared control of human intervention}
For firefighting tasks, we consider the following scenario where human intervention can make the MRS more efficient to reduce the damage caused by fire. There are five fire sources, several circle obstacles, and a bar barrier in the environment. The initial setup of these experiments is shown in TABLE \ref{setup}. The fire source radius increases with time (1m/s), and the time required to extinguish the fire is proportional to the size of the fire; the area covered by the $i$-th fire source is denoted by the loss area $a(s_i)$. The ultimate goal of the MRS in extinguishing the fire is to minimize the total loss area.

\begin{table}[]
	\centering
	\caption{Initial setup of simulation experiments.} 
	\resizebox{\linewidth}{!}{
		\begin{tabular}{l|l}
			\toprule
			Robot(Position)          &           (700,300)   \ \ (670,350)    \ \ (730,350)                                            \\ \midrule
			Fire source(Position)[Size] & \begin{tabular}[l]{@{}l@{}}(500,250)[20] \ \  (300,450)[30] \ \ (170,120)[50]\\(650,600)[10] \ \             (1100,300)[40]\end{tabular} \\ \midrule
			Obstacles(Position)[Size]                & \begin{tabular}[l]{@{}l@{}}(300,200)[20] \ \  (670,450)[20] \ \ (480,400)[20]      \\   (150,450)[20] \ \   (650,140)[20]     \end{tabular}                             \\ \bottomrule
		\end{tabular}
	}
	\label{setup}
\end{table}

\textbf{Case 1}: MRS changes target priority with human intervention. 
The robots choose the next priority fire to be extinguished based on their observations.
The MRS will maintain the formation and make adjustments when one robot encounters obstacles while performing the task. 
During the firefighting process, the robots are distributed to the appropriate locations according to the size of the fire. 
After extinguishing fire source 1, the robots choose fire source 2 over the actual more threatening fire source 3 due to the obstruction caused by obstacle 1. 
When humans intervene, the robot will be guided to change its trajectory to discover and select the fire source 3, as shown in Fig. \ref{human_intervention}.
The results in Table \ref{comparison} show that the area of damage caused by the fire is reduced after human intervention.

\textbf{Case 2}: MRS finds undetected target with human intervention. 
After the robots extinguish fire source 2, there is still a fire source left in the environment. 
Because of the barriers, the robot could not detect the fire until the fire spreads further, which would cause more damage. 
At this point, human intervention can guide the robot to detect the target early and reduce the damage, as shown in Table \ref{comparison}.

\begin{table}[]
	\centering
	\caption{comparison of the loss areas with different situations.}
	\begin{tabular}{@{}l|l|l|l@{}}
		
		\toprule
		Loss area & \multicolumn{1}{c}{Without human}                                    & \multicolumn{2}{c}{With human}                                                                                                        \\ \midrule
		& \begin{tabular}[c]{@{}l@{}}Robot autonomous\\ operation\end{tabular} & \begin{tabular}[c]{@{}l@{}}Change target \\ priority\end{tabular} & \begin{tabular}[c]{@{}l@{}}Find undetected \\ target\end{tabular} \\\midrule
		$a(s_1)$         & 3512.61                                                              & 3512.61                                                           & 3512.61                                                           \\
		$a(s_2)$         & 7026.29                                                                 & 8639.3                                                               & 7026.29                                                              \\
		$a(s_3)$         & 16203.8                                                                  & 14491.39                                                              & 16203.8                                                                \\
		$a(s_4)$         & 4646.9                                                                  & 4273                                                             & 4646.9                                                             \\
		$a(s_5)$         & 20157.5                                                                  & 17335.6                                                                & 17920.2                                                              \\
		Sum       & 51547.1                                                                   & \textbf{48253.9}                                                                & \textbf{49309.8}                                                                \\ \bottomrule
	\end{tabular}
	\label{comparison}
\end{table}

	\subsection{Human-computer interaction methods}
	As we mentioned earlier, we assume that in a scenario of human-multi-robot collaboration, humans obtain a global view through the UAV, thus guiding the robots to accomplish the task together using multiple human-robot interactions. We briefly overview the brain-computer interfaces, myoelectric wristbands, and eye-tracking devices used.
	\subsubsection{brain-computer interface}
	The brain-computer interface (BCI) is a unique form of human-computer interaction that uses brain activity to communicate with the external world \cite{wolpaw2002brain}. Researchers typically capture brainwave signals using non-invasive electrodes placed on the scalp. Unlike traditional methods, BCI generates commands directly from EEG signals, reflecting users' choices and uncertainties \cite{lotte2018areview}. Therefore, BCI is ideal for controlling intelligent robots.
	
	Our experiment uses the NeuSen W wireless EEG system, which is portable, stable, and well-shielded, equipped with a nine-axis motion sensor for both controlled and natural experiments.
	The hardware includes a 64-channel EEG cap and an amplifier, with electrodes positioned according to the international 10-20 system.
	\subsubsection{myoelectric wristbands}
	The EMG interface captures electromyography signals from the skin to interpret human intent. Surface-measured muscle contractions range from tens to hundreds of microvolts. Although weak and prone to noise, EMG signals directly reflect body activity, making them an advanced HCI for predicting user intentions \cite{atzori2014electromyography}. Extensive research has utilized EMG to control robots \cite{wolf2013gesture,kundu2018hand,luo2020a,suresh2020human}.
	
	We have developed an EMG acquisition device that contains an array of standard medical snap buttons compatible with dry Ag/AgCl and gel electrodes. These snap buttons are arranged longitudinally in pairs along the arm to detect distinct muscle activities. 
	Our wristband supports eight channels for surface EMG acquisition. This multi-channel design improves sensitivity and fidelity, enabling the detection of subtle hand gestures and fine motions.
	
	\subsubsection{eye-tracking device}
	The eye-tracker monitors and records human eye movements using infrared light and camera technology. This data can be analyzed to understand users' intentions, predict future actions, and gain insights into their cognitive states \cite{koochaki2018predicting}. Eye-tracking techniques have been widely used to recognize human intentions for indirect input or task assistance \cite{shafiei2024development,aronson2018eye,ming2016anticipatory}.
	
	Intent prediction through eye movements is natural and seamless, serving as an efficient interface. Our system uses a simple Tobii eye-tracker to capture eye movements. The eye-tracker captures gaze positions over time, allowing us to derive gaze trajectories and construct heat maps that illustrate attention distribution and estimate cognitive workload.
	

	\subsection{Physical Multi-robot Experiment}
	Based on our self-designed human-multi-robot interaction system, we conduct physical robotics experiments. 
	All participants were provided with written informed consent to participate in the experiment in accordance with the Declaration of Helsinki.
	We consider the following two scenarios:
	(1) The desired path is given in advance, and the robot moves along the path motion and deviates from the path when the target is found. 
	(2) While the robot is executing the task, a human intervenes to change the execution sequence.
	\subsubsection{case 1}
	When the robot cannot directly observe the target, it is slow to search for it by directly exploring the environment. At this point, the human operator can utilize the information provided by the UAV to guide the robot's movement, given the trajectory, which can improve the efficiency of the task.
	Consider a scenario like the one in Fig. \ref{physical_case1}, where the environment has circular and bar-shaped obstacles, simulating objects such as general obstacles and walls in real situations.The robot should avoid collisions with obstacles as well as other robots at all times during the task. Besides, it is assumed that the robot has a limited perceptual range and can only detect the target when it is close to it. The red line is the human’s given path. When the robots find the target, they break away from the path, and when the task is completed, the robots can still be guided to the given path.
	\begin{figure}[h!]
		\includegraphics[width=\linewidth]{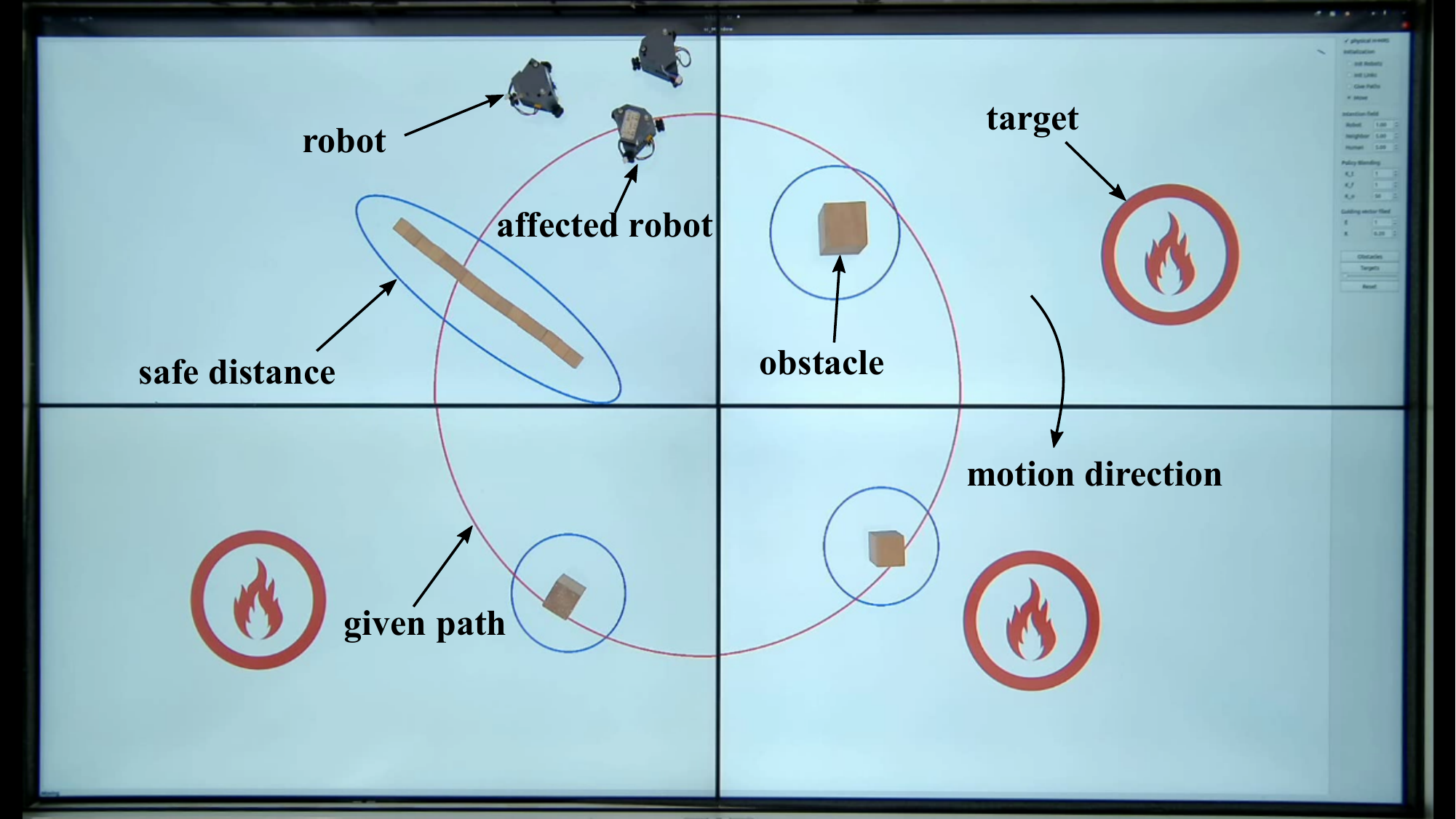}
		\caption{physical experimental scenario for case 1.} 
		\label{physical_case1}
	\end{figure}
	
	The experimental result is shown in Fig. \ref{physical_result1}. The robots
	achieve the obstacle avoidance requirement throughout, even when the desired path crosses barriers. Besides, The robots are able to follow the desired path and reach the target position for operation.
	\begin{figure}[h!]
	\includegraphics[width=\linewidth]{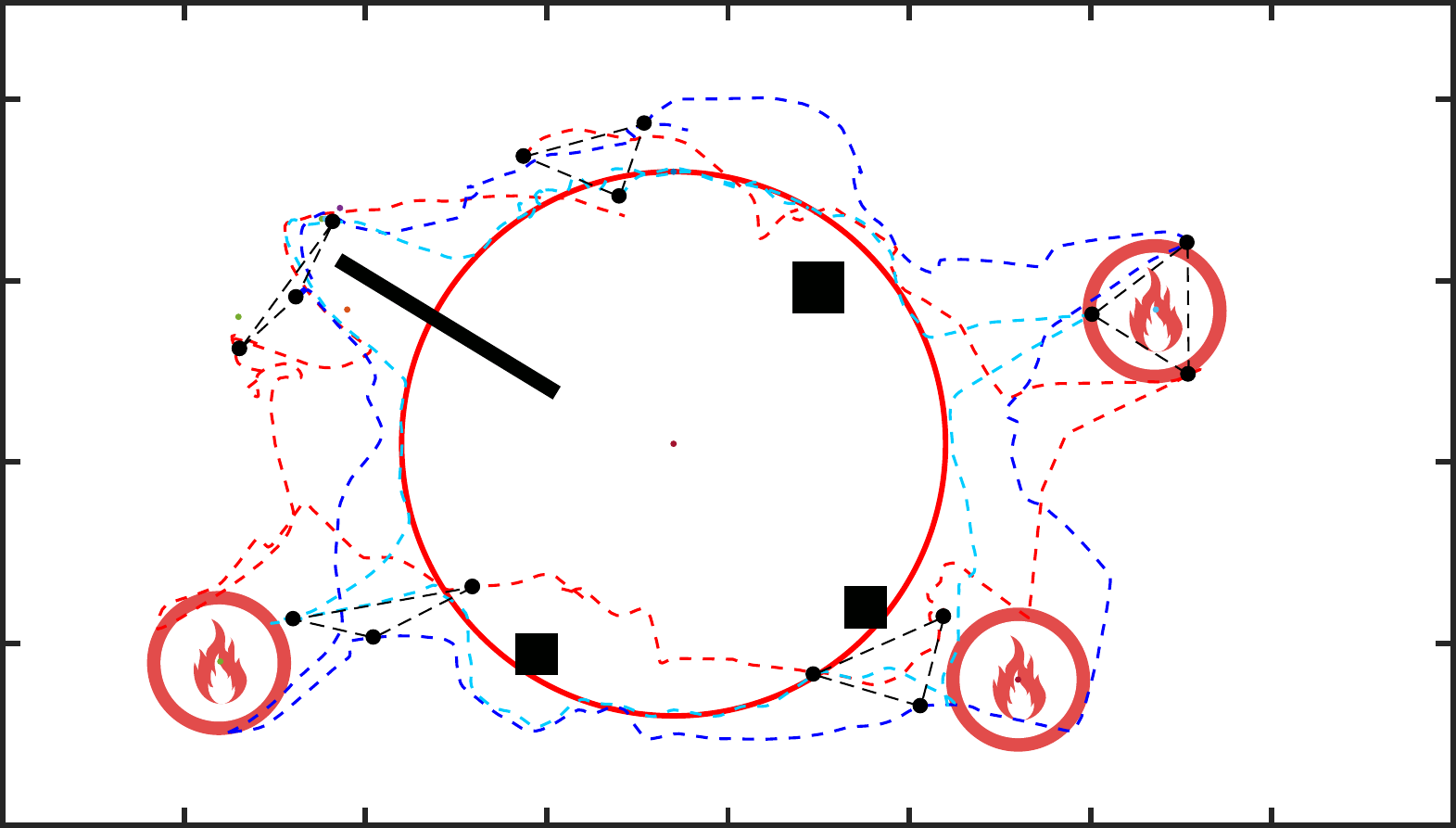}
	\caption{Robot trajectories of case 1.} 
	\label{physical_result1}
\end{figure}
	\subsubsection{case 2}
	A human operator intervenes online in real-time with a robot based on a human-multi-robot interaction interface. In this case, the eye-tracker is used to capture the gaze to draw a desired path, the BCI interface is used to select the robot to be guided, and gestures are used to switch the task progress. The experimental equipment is shown in Fig. \ref{test}. Although the operator in the picture is holding the drone remote control in his hand, it is only to simulate that the operator can use tiny gestures for ground robot control even though his hands are manipulating the drone. In fact, we do not use the remote control for the whole experiment. It is worth noting that the entire task process does not require using
	traditional interaction methods such as mouse and keyboard, which provides a flexible and feasible solution for human-multi-robot collaboration.
	
	\begin{figure}[h!]
		\includegraphics[width=\linewidth]{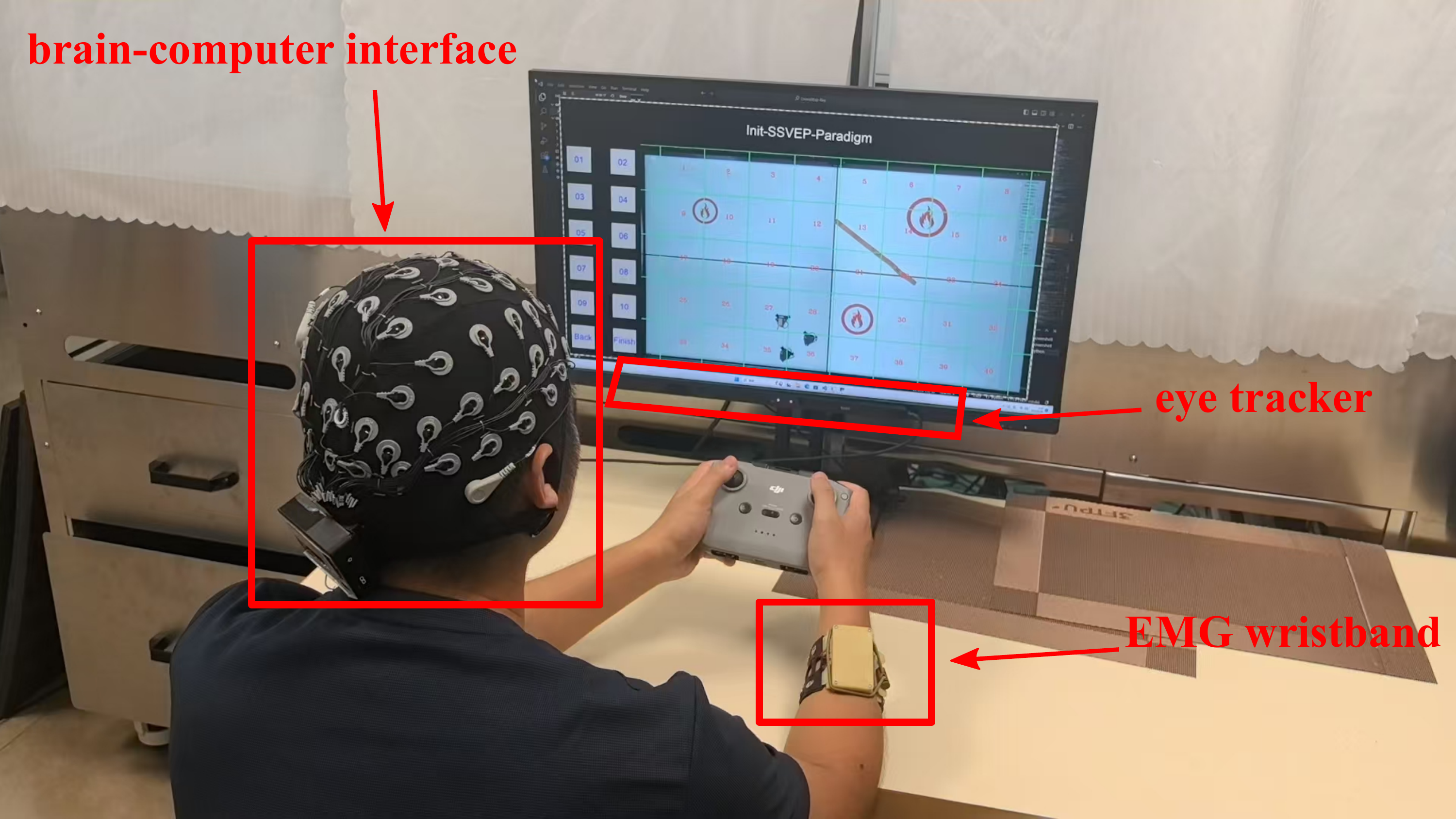}
		\caption{Various HCI devices used in the experiment.} 
		\label{test}
	\end{figure}
	\begin{figure}[h!]
		\includegraphics[width=\linewidth]{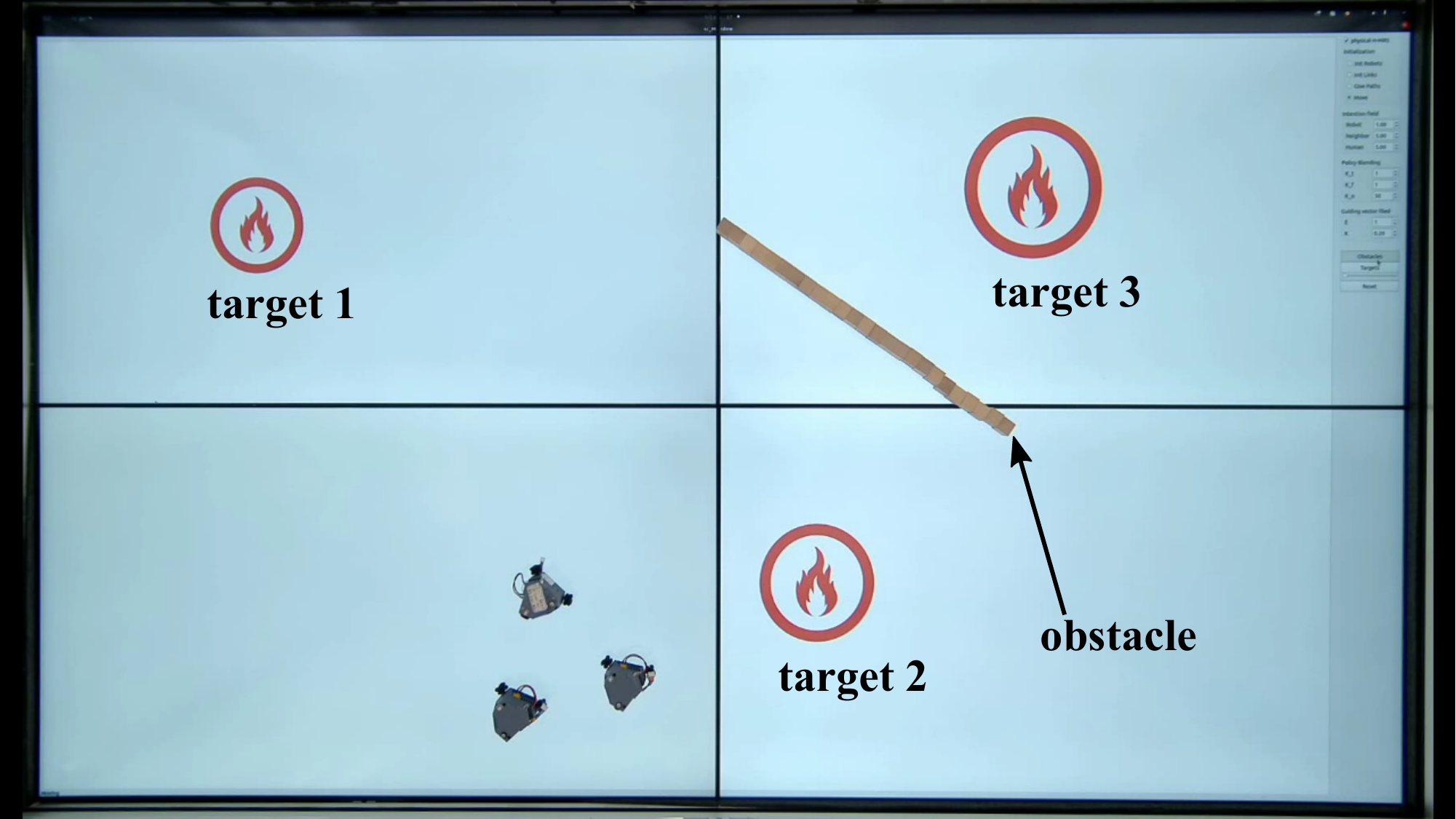}
		\caption{Physical experimental scenario for case 2.} 
		\label{physical_case2}
	\end{figure}
	
The experimental scenario is shown in Fig. \ref{physical_case2}. In this case, the robots are not able to observe targrt 3, which has the highest threat level, due to the occlusion of the barrier.
The robots operate autonomously with an execution sequence target 2 $\rightarrow$ target 1 $\rightarrow$ target 3, which leads to great damage.
At this point, the human operator guides the robot around the obstacle to prioritize the elimination of target 3. Fig. \ref{physical_result}(a)-(f) show the motion of the robot in sequence.
The direction of the red arrow in the figure represents the robot’s current direction of motion. 
In Fig. \ref{physical_result}(a), the robots are close to target 2, and target 2 is a greater threat than target 1, so the robots prioritize movement toward target 2. When target 2 is eliminated, the robots can only observe target 1 due to the obstacles, so the robots autonomously move towards target 1 (see Fig. \ref{physical_result}(b)). However, in reality, the threat level of target 3 is greater.
At this point, the human operator guides the robot along the desired path by first pausing the robot’s motion through hand gestures, giving the desired path through eye
movements, and selecting the affected robot through a brain-computer interface (see Fig. \ref{physical_result}(c)). 
When the robots move to the end of the desired path, the robots move toward target 2 (see Fig. \ref{physical_result}(d)), and later move towards target 1 avoiding obstacles. The robot trajectories are show in Fig. \ref{physical_result_2path}.

\begin{figure}[h!]
	\centering
	\subfloat[]{\includegraphics[width=0.47\linewidth]{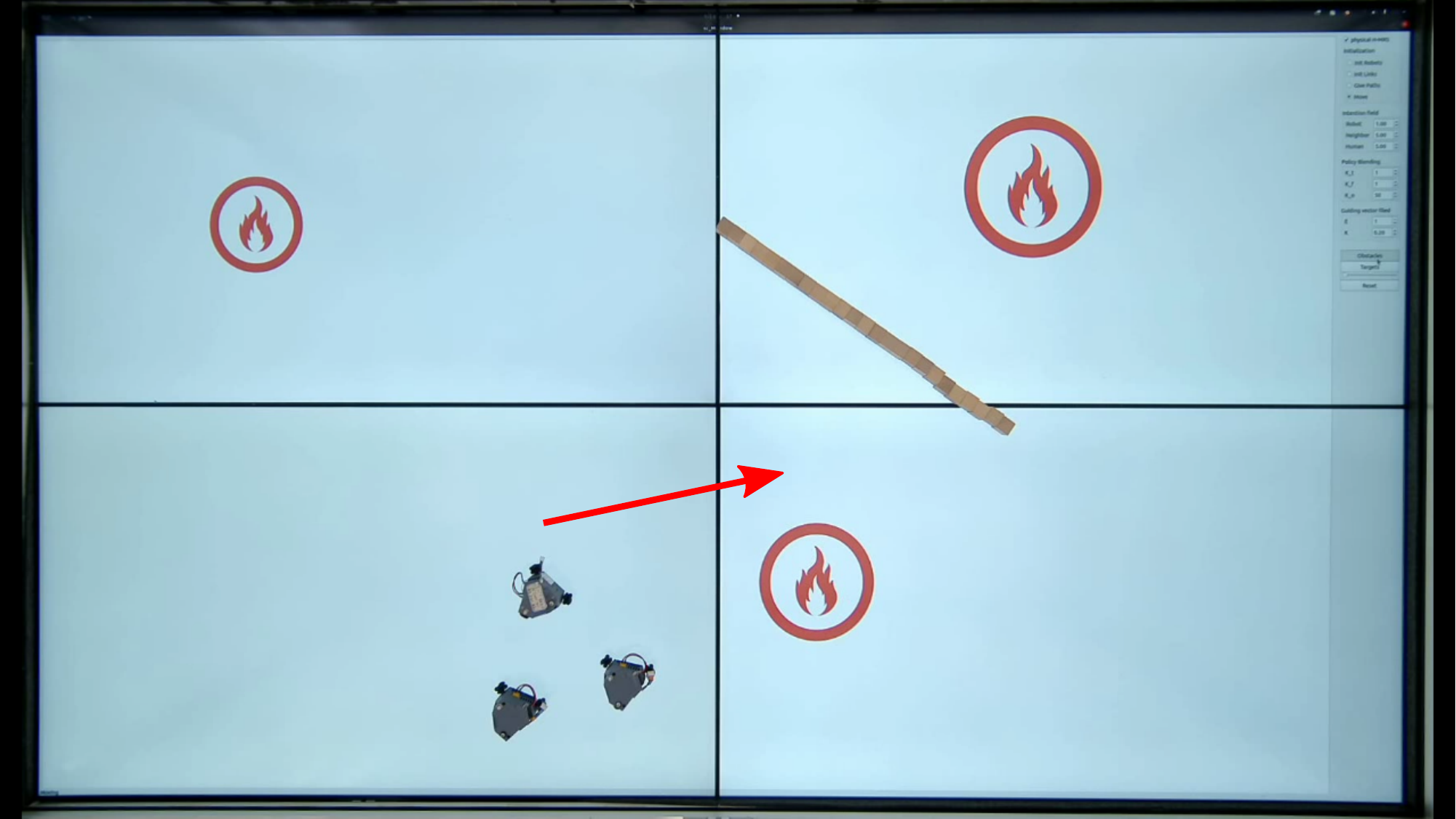}}
	\hspace{1mm}
	\subfloat[]{\includegraphics[width=0.47\linewidth]{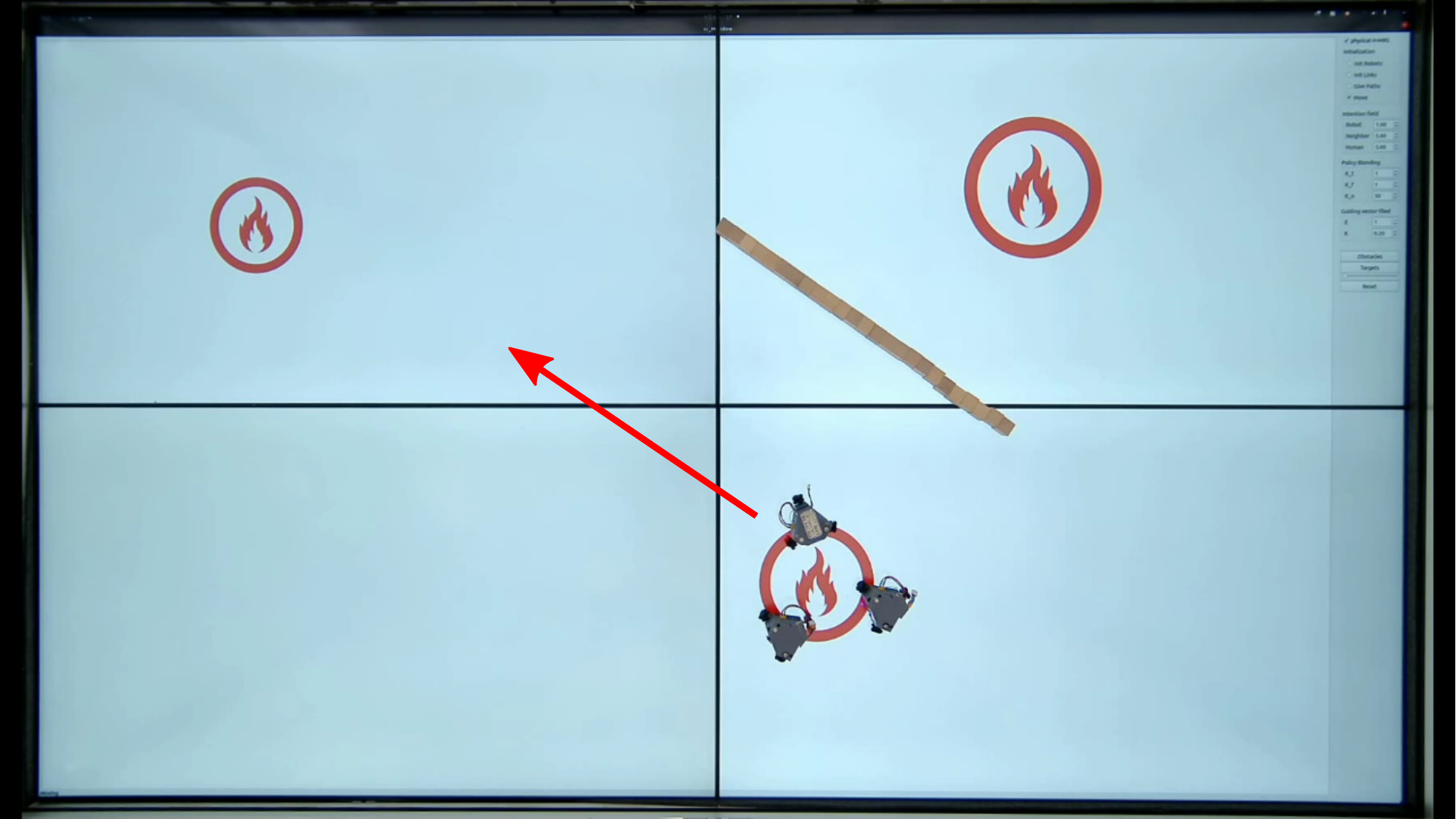}}\\
	\subfloat[]{\includegraphics[width=0.47\linewidth]{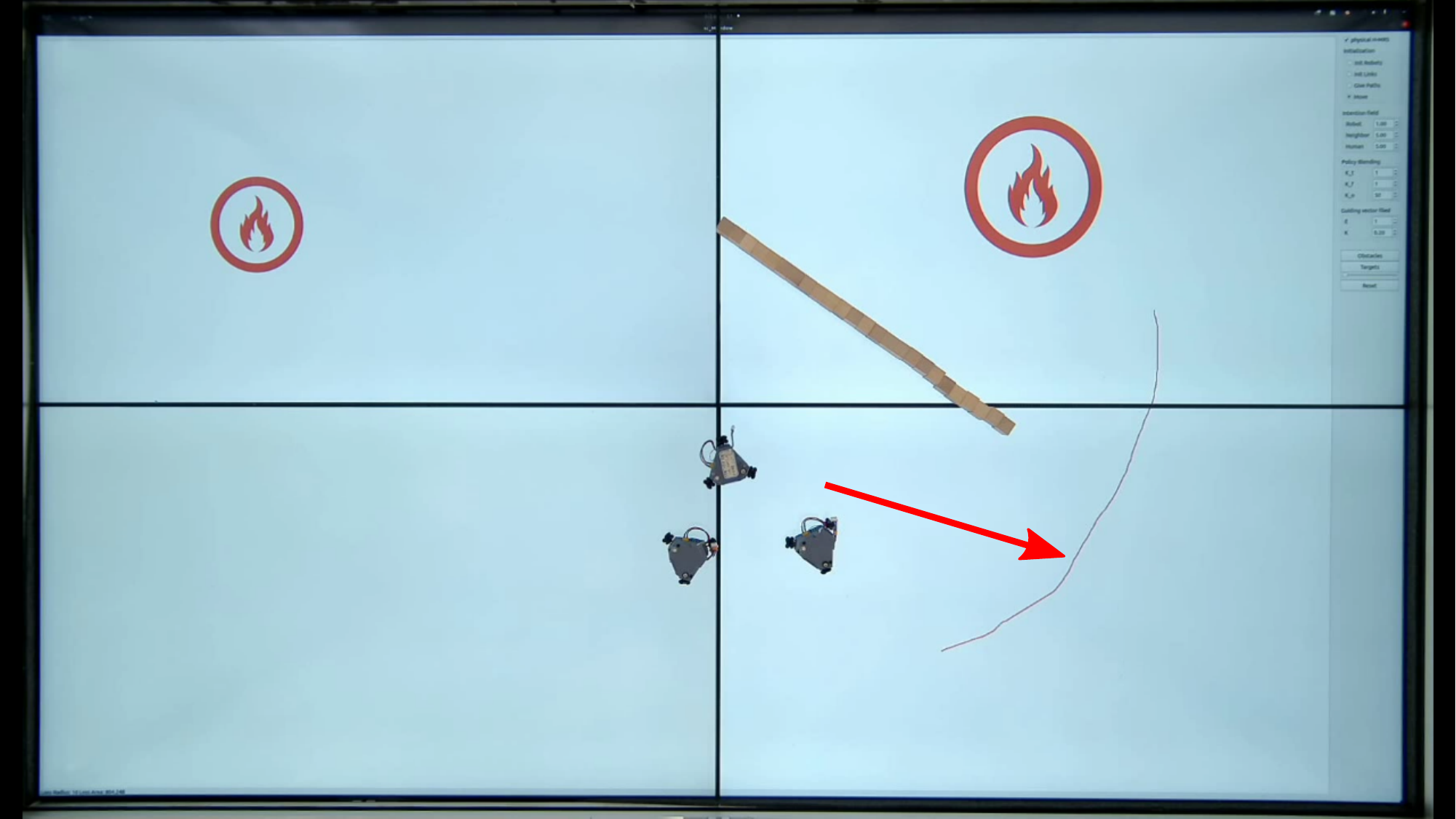}}
	\hspace{1mm}
	\subfloat[]{\includegraphics[width=0.47\linewidth]{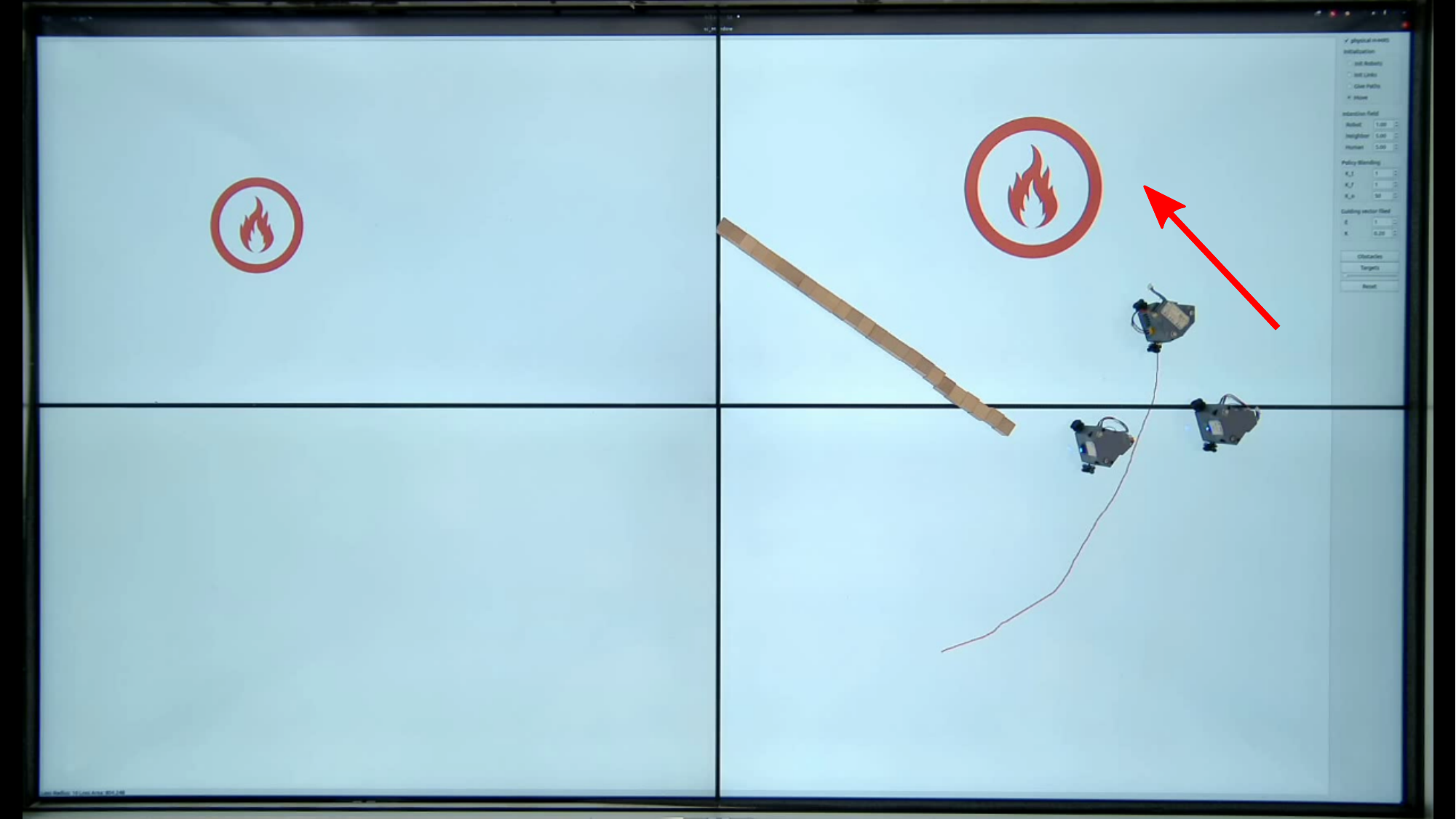}}\\
	\subfloat[]{\includegraphics[width=0.47\linewidth]{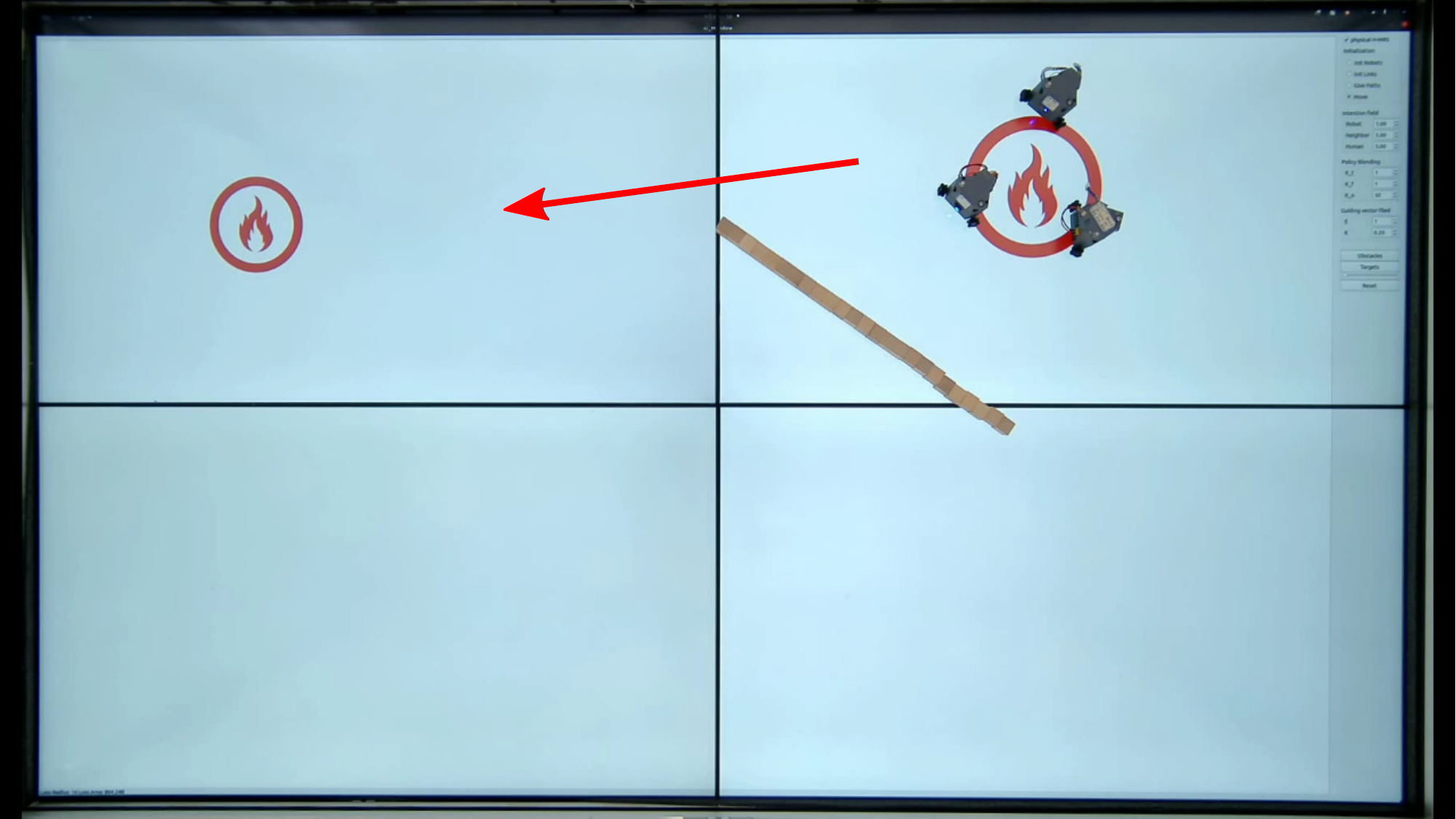}}
	\hspace{1mm}
	\subfloat[]{\includegraphics[width=0.47\linewidth]{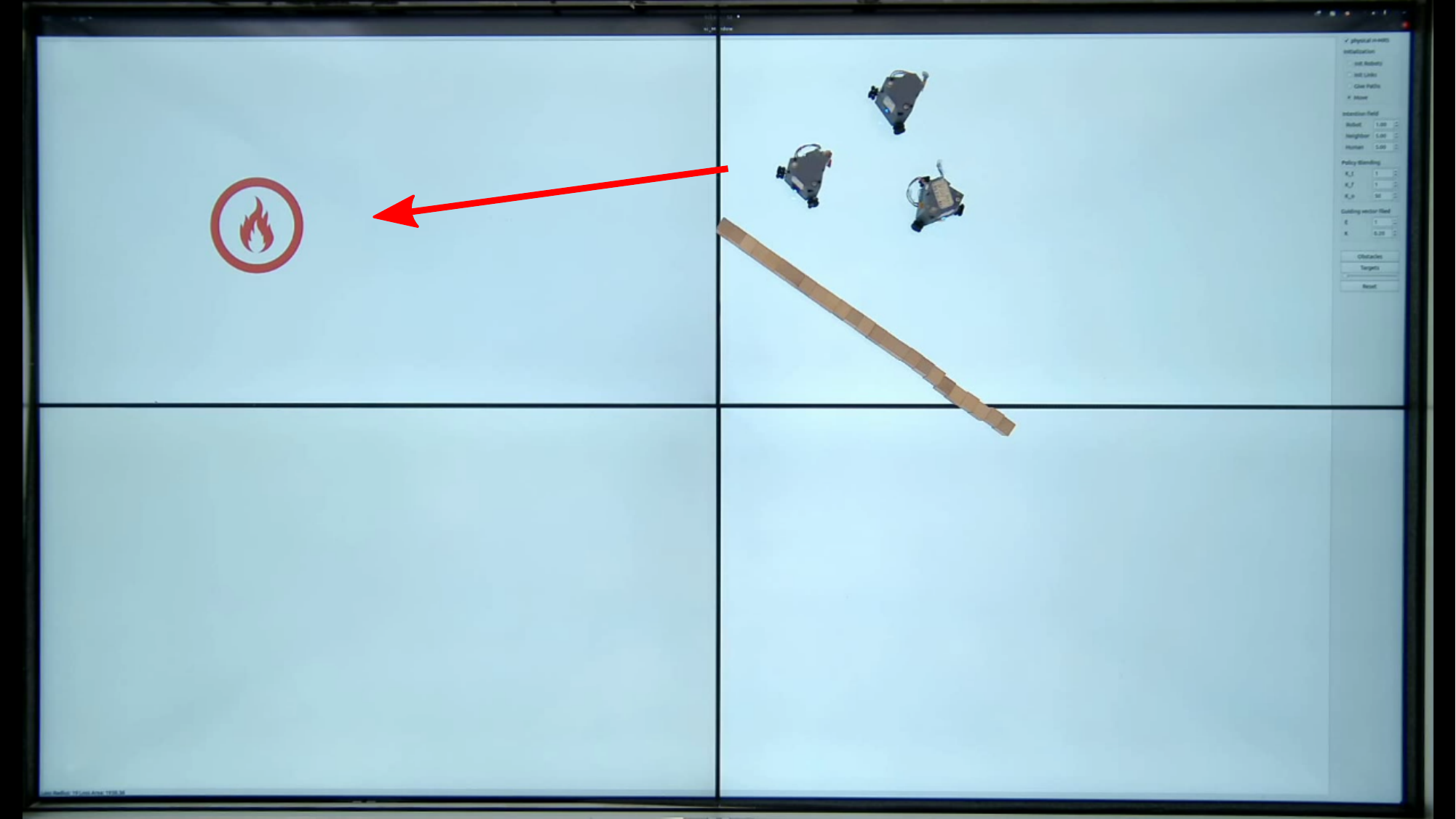}}
	\caption{The execution process with human intervention.} 
	\label{physical_result}
\end{figure}

\begin{figure}[h!]
	\includegraphics[width=\linewidth]{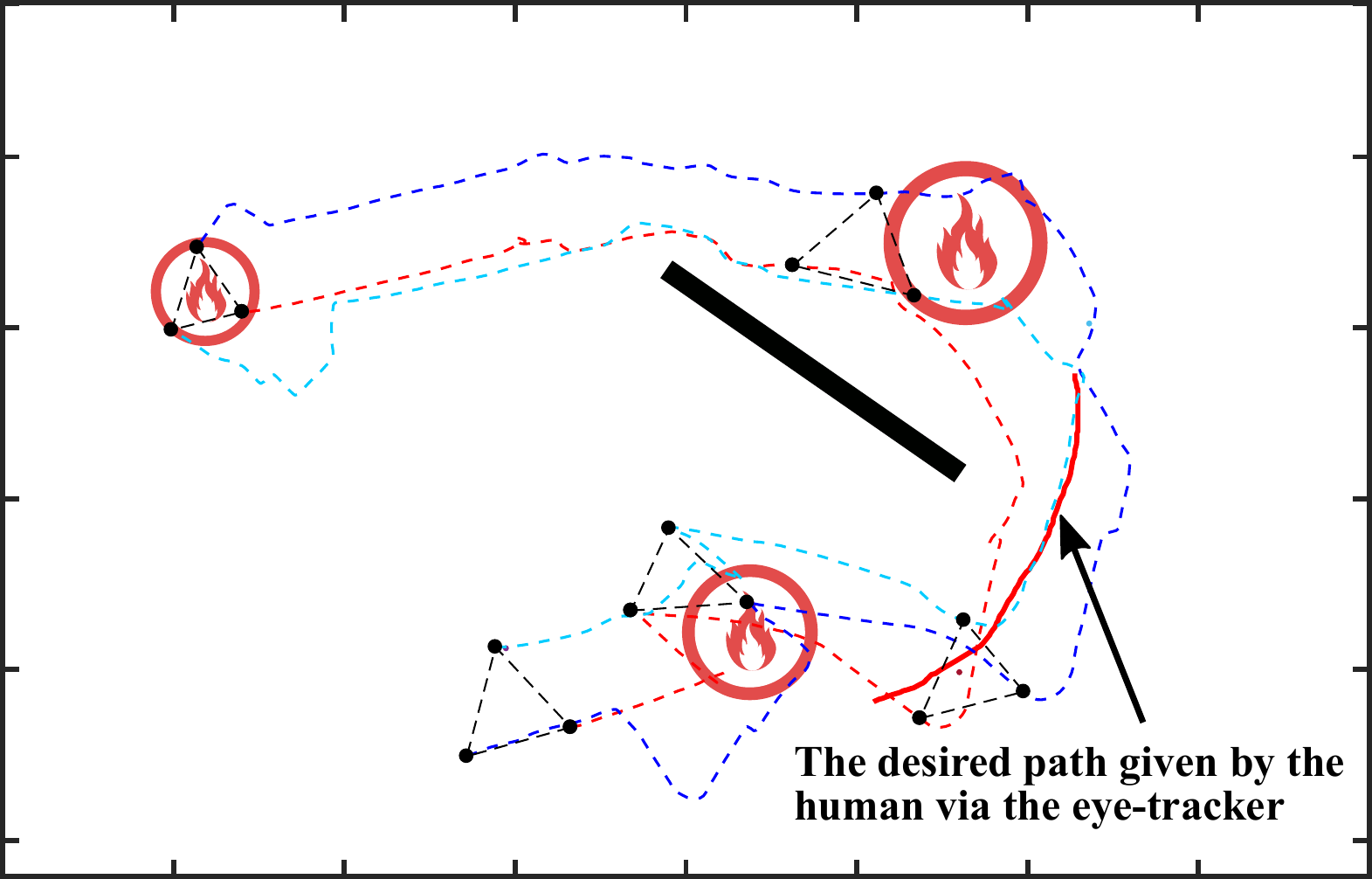}
	\caption{Physical experimental scenario for case 2.} 
	\label{physical_result_2path}
\end{figure}

\section{Conclusion}
Multi-robot teams can robustly accomplish large-scale complex missions.
However, robots often need human involvement to solve tough dilemmas. 
In this context, a shared control method integrating human intentions within a multi-robot team can improve the execution capabilities of multi-robot systems.
In this paper, a novel shared control method is proposed that allows humans to use desired paths (or certain areas) as inputs without requiring real-time human involvement in the motion control of robots, thus reducing the burden on human operators.
In our method, with human-influenced guiding vector fields utilized to describe human intentions, a hierarchical shared control framework is proposed to propagate and merge the robot and human intentions.
In addition, a novel human-multi-robot interaction system based on multi-touch screens is built for a physical demonstration to validate the proposed algorithm.
Besides, we use a variety of human-robot interactions (brain-computer interface, myoelectric wristband, eye-tracking) to intervene in the robot’s task progression. The
results in different scenarios show that our method promotes the effectiveness and efficiency for completing the tasks.


	\clearpage
	\section*{Appendix}
	\subsection{The stability analysis of the intentional field}
	\label{stability_in}
	Consider the robot as a first-order model, i.e., $\dot{\boldsymbol{x}}_i=\boldsymbol{v}_i$.
	Equation \eqref{shared_intention1} can be rewritten as follows:
	\begin{equation}
	\begin{aligned}
	\boldsymbol{v}_i^s(k+1)= & - \omega_0\boldsymbol{v}_i^s(k) +\omega_1\left(\boldsymbol{v}_i^t(k)-\boldsymbol{v}_i^s(k)\right) \\
	& +\omega_2 \sum_{j \in \mathcal{N}_i} \Phi\left(\boldsymbol{v}_j^s(k)-\boldsymbol{v}_i^s(k)\right) \\
	& +\omega_3 \Psi_i\left(\boldsymbol{v}_i^h(k)-\boldsymbol{v}_i^s(k)\right).
	\end{aligned}
	\label{shared_intention}
	\end{equation}
	For the whole robot formation, the state can be represented as $\boldsymbol{x}=[\boldsymbol{x}_1^\top,\cdots,\boldsymbol{x}_N^\top]^\top\in \mathbb{R}^{2N}$.
	Let $\Psi=\text{diag}({\Psi_i}) \otimes I_2  \in \mathbb{R}^{2N \times 2N}$, $\Phi(\boldsymbol{x})=[\phi(\boldsymbol{x}_1)^\top,\cdots,\phi(\boldsymbol{x}_N)^\top]^\top$, $\boldsymbol{v}_s=[{\boldsymbol{v}^s_1}^\top,\cdots,{\boldsymbol{v}^s_N}^\top]^\top$, $\boldsymbol{v}_t=[{\boldsymbol{v}^t_1}^\top,\cdots,{\boldsymbol{v}^t_N}^\top]^\top$, $\boldsymbol{v}_h=[{\boldsymbol{v}^h_1}^\top,\cdots,{\boldsymbol{v}^h_N}^\top]^\top$, where $\otimes$ is the Kronecker product.
	Then, \eqref{shared_intention} becomes
	\begin{equation}
	\boldsymbol{v}_s= -\omega_0 \boldsymbol{v}_s  + \omega_1 (\boldsymbol{v}_t-\boldsymbol{v}_s) - \omega_2 \widetilde{D}\Phi(\widetilde{D}^\top \boldsymbol{v}_s)+\omega_3 \Psi(\boldsymbol{v}_h-\boldsymbol{v}_s),
	\end{equation}
	where $\widetilde{D} = D \otimes I_2$, D is the incidence matrix.
	We use the Lyapunov method for stability analysis of the intention field model.
	\begin{theorem}
		The intention field model with $\boldsymbol{v}_t$ and  $\boldsymbol{v}_s$ as the input, and $\boldsymbol{v}_s$ as the state, is input-to-state stable.
		Moreover, let $\Vert\boldsymbol{v}_s\Vert_\infty = \limsup \limits_{t\to\infty} \Vert \boldsymbol{v}_s(t) \Vert _2$, $\gamma_t=\frac{1}{2}\sqrt{\frac{\omega_1}{\omega_0}}$, and $\gamma_h=\frac{1}{2}\sqrt{\frac{\omega_3}{\omega_0}}$, we can obtain 
		\begin{equation}
		\Vert \boldsymbol{v}_s \Vert_\infty \leq  \gamma _t  \Vert \boldsymbol{v}_t \Vert _\infty + \gamma _h  \Vert \boldsymbol{v}_h \Vert _\infty.
		\label{intent_model11}
		\end{equation}
	\end{theorem}
	\begin{proof}
		Let $V(\boldsymbol{v}_s)= \frac{1}{2}\boldsymbol{v}_s^\top\boldsymbol{v}_s$, from \eqref{shared_intention}, we can obtain 
		\begin{equation}
		\begin{aligned}
		\dot{V}= &-\omega_0 \Vert \boldsymbol{v}_s \Vert ^2 + \omega_1 \boldsymbol{v}_s ^\top \Psi (\boldsymbol{v}_t-\boldsymbol{v}_s) \\ 
		&- \omega_2 \boldsymbol{v}_s ^\top \widetilde{D}\Phi(\widetilde{D}^\top \boldsymbol{v}_s) + \omega_3 \boldsymbol{v}_s ^\top \Psi (\boldsymbol{v}_h-\boldsymbol{v}_s).
		\end{aligned}
		\label{li1}
		\end{equation}
		More generally, suppose robots $1, \cdots, N$, all have their own target, thus from {\small $\boldsymbol{v}^s_i {^\top}\boldsymbol{v}^t_i \leq \Vert \boldsymbol{v}^s_i \Vert ^2 + \frac{1}{4}\Vert \boldsymbol{v}^t_i \Vert^2$}, we have 
		\begin{equation}
		\boldsymbol{v}_s ^\top \Psi (\boldsymbol{v}_t-\boldsymbol{v}_s)=\sum_{i=1}^N (\boldsymbol{v}_i^s {^\top} \boldsymbol{v}^t_i-\Vert \boldsymbol{v}^s_i\Vert^2) \leq \frac{\Vert\boldsymbol{v}_t \Vert^2}{4}.
		\end{equation}
		Likewise, we have 
		\begin{equation}
		\boldsymbol{v}_s ^\top \Psi (\boldsymbol{v}_h-\boldsymbol{v}_s)=\sum_{i=1}^N (\boldsymbol{v}^s_i {^\top} \boldsymbol{v}^h_i-\Vert \boldsymbol{v}^s_i\Vert^2) \leq \frac{\Vert\boldsymbol{v}_h \Vert^2}{4}.
		\end{equation}
		Since $\boldsymbol{v}_s ^\top \widetilde{D}\Phi(\widetilde{D}^\top \boldsymbol{v}_s) \geq 0$, from \eqref{li1} we have 
		\begin{equation}
		\begin{aligned}
		\dot{V} &\leq -\omega_0 \Vert \boldsymbol{v}_s\Vert^2+\frac{1}{4}\omega_1\Vert \boldsymbol{v}_t \Vert^2+\frac{1}{4}\omega_3\Vert \boldsymbol{v}_h \Vert^2 \\ &\leq -\omega_0 \Vert \boldsymbol{v}_s\Vert^2+\frac{1}{4}(\sqrt{\omega_1}\Vert \boldsymbol{v}_t \Vert+\sqrt{\omega_3}\Vert \boldsymbol{v}_h \Vert)^2.
		\end{aligned}
		\end{equation}
		From this, we can get that the intention model is input-to-state stable when $    \Vert \boldsymbol{v}_s \Vert_\infty \leq  \frac{1}{2}\sqrt{\frac{\omega_1}{\omega_0}}  \Vert \boldsymbol{v}_t \Vert _\infty + \frac{1}{2}\sqrt{\frac{\omega_3}{\omega_0}}  \Vert \boldsymbol{v}_h \Vert _\infty.$
	\end{proof}

	\subsection{The stability analysis of the consensus network}
	\label{stability_con}
	According to the consensus theory \cite{lewis2013cooperative}, the system dynamics of the consensus network can be represented as
	\begin{equation}
	\boldsymbol{\dot{x}}=-\widetilde{L} \boldsymbol{x}+\Delta 
	\end{equation}
	\begin{equation}
	\boldsymbol{\dot{e}}=-\widetilde{D}^\top \widetilde{D} \boldsymbol{e}+\widetilde{D}^\top \Delta
	\label{error}
	\end{equation}
	where $\boldsymbol{x}=[\boldsymbol{x}_1^\top,\cdots,\boldsymbol{x}_N^\top]^\top\in \mathbb{R}^{2N}$ is the state set; $\Delta=[\Delta_1^\top,\cdots,\Delta_N^\top]$, $\Delta_i=\boldsymbol{v}_i-\boldsymbol{v}_i^f=\lambda_i(\hat{\boldsymbol{v}}_i^s-\boldsymbol{v}_i^f)$; $\widetilde{L}=L \otimes I_2$, $L$ is the Laplacian matrix; $\boldsymbol{e}=\widetilde{D}^\top \boldsymbol{x}$ is the consensus error.
	And we can easily obtain Lemma 1.
	\begin{lemma}
		Given a connected graph $\mathcal{G}(\mathcal{V},\mathcal{E})$, $\widetilde{D}^\top\widetilde{D}$ is positive definite on space $\mathrm{span}\{\widetilde{D}^\top\}$.
		Therefore, for $\forall \boldsymbol{e},\exists \delta >0$, such that $\boldsymbol{e}^\top\widetilde{D}^\top\widetilde{D}\boldsymbol{e} \geq \delta \|\boldsymbol{e}\|^2$.
	\end{lemma}
	
	For the designed $\lambda$ function in \eqref{lambda}, we have Lemma 2.
	\begin{lemma}
		For $\lambda(a,b)$, there exists $\bar{f}_b(a) \in \mathcal{K}$, when $b>\bar{f}_b(a)$ such that $\forall a > 0, \forall c \geq 0$,
		\begin{equation}
		\lambda(a,b)(b+c)b \leq \frac{1}{2}b^2.
		\end{equation}
	\end{lemma}
	\begin{proof} 
		Let $g_a(b)=\lambda(a,b)(1+\frac{c}{b})$, then we can easily obtain that $g_a(b)$ is strictly decreasing in $b$, and $\lim_{b\rightarrow \infty}g_a(b)=0$, $\lim_{b\rightarrow 0}g_a(b)=\infty$.
		Therefore, for $\forall a>0$, there exists $b^*$ such that $g_a(b^*)=\frac{1}{2}$, and for $\forall b \geq b^*$, $g_a(b)\leq\frac{1}{2}$.
	\end{proof}
	
	\begin{theorem}
		The consensus network with $\hat{\boldsymbol{v}}_s$ as the input, and $\boldsymbol{e}$ as the state, is input-to-state stable.
		Moreover, let  $\Vert \boldsymbol{e} \Vert_\infty = \limsup \limits_{t\to\infty} \Vert \boldsymbol{e}(t) \Vert $, $\gamma_e=\frac{1}{2}\sqrt{\frac{\omega_1}{\omega_0}}$, we can obtain 
		\begin{equation}
		\Vert \boldsymbol{e} \Vert_\infty \leq  \gamma _e  \Vert \hat{\boldsymbol{v}}_s \Vert _\infty.
		\label{intent_model}
		\end{equation}
	\end{theorem}
	\begin{proof}
		Let $V(\boldsymbol{e})=\frac{1}{2}\boldsymbol{e}^\top\boldsymbol{e}$, from \eqref{error}, we can obtain
		\begin{equation}
		\dot{V}(e)=-\boldsymbol{e}^\top \widetilde{D}^\top \widetilde{D} \boldsymbol{e}+\boldsymbol{e}^\top \widetilde{D}^\top \Delta=-\|\boldsymbol{v}_f\|^2-\boldsymbol{v}_f^\top \Delta,
		\label{li}
		\end{equation}
		where $\boldsymbol{v}_f=-\widetilde{D}\boldsymbol{e}=[{\boldsymbol{v}_1^f}^\top,\cdots,{\boldsymbol{v}_N^f}^\top]^\top$.
		
		From \eqref{final_s}, we can obtain $\hat{\boldsymbol{v}}_i^s \equiv \mathcal{C}$. Since $\Delta_i=\boldsymbol{v}_i-\boldsymbol{v}_i^f=\lambda_i(\hat{\boldsymbol{v}}_i^s-\boldsymbol{v}_i^f)$, we obtain that
		\begin{equation}
		-\boldsymbol{v}_f^\top \Delta=-\sum_{i \in \mathcal{V}} {\boldsymbol{v}_i^f}^\top \Delta_i \leq \sum_{i \in \mathcal{V}} \lambda_i (\mathcal{C}+ \| \boldsymbol{v}_i^f \| ) \| \boldsymbol{v}_i^f \|.
		\end{equation}
		
		From Lemma 2, it follows that $\forall i \in \mathcal{V}$,
		
		\begin{equation}
		-{\boldsymbol{v}_i^f}^\top \Delta_i \leq\left\{\begin{array}{ll}
		\|\boldsymbol{v}_i^f\|^2 / 2 & \text { if } \|\boldsymbol{v}_i^f \| \geq \bar{f}_b\left(\| \boldsymbol{v}_i^s \|\right) \\
		\left(\mathcal{C}+\bar{f}_b\left(\| \boldsymbol{v}_i^s \|\right) \bar{f}_b\left(\left\|\boldsymbol{v}_i^s\right\|\right)\right. & \text { if }\|\boldsymbol{v}_i^f\|<\bar{f}_b\left(\|\boldsymbol{v}_i^s\|\right)
		\end{array}\right.
		\end{equation}
		
		Let $\hat{f}_s=\left(\mathcal{C}+\bar{f}_b\right) \bar{f}_b$, then $\hat{f}_s \in \mathcal{K}_{\infty}$. Then it follows that $\forall i \in \mathcal{V},-{\boldsymbol{v}_i^f}^\top \Delta_i \leq \|\boldsymbol{v}_i^f \|^2 / 2+\hat{f}_s\left(\left\|\boldsymbol{v}_i^s\right\|\right)$. Since $\sum_{i \in \mathcal{V}} \hat{f}_s\left(\|\boldsymbol{v}_i^s \|\right) \leq$ $N \hat{f}_s(\|\boldsymbol{v}_s\|)$, from \eqref{li} and Lemma 1, we have
		\begin{equation}
		\begin{aligned}
		\dot{V}(\boldsymbol{e}) &=-\|\boldsymbol{v}_f\|^2-\boldsymbol{v}_f^\top \Delta\\ &\leq-\frac{1}{2}\|\boldsymbol{v}_f\|^2+\hat{f}_s(\|\boldsymbol{v}_s\|) \leq-\frac{1}{2} \delta \|\boldsymbol{e}\|^2+\hat{f}_s(\|\boldsymbol{v}_s\|).
		\end{aligned}
		\end{equation}
		
		Thus the system is input-to-state stable, with a bound that $\|\boldsymbol{e}\|_{\infty} \leq\left(2 N \hat{f}_s\left(\|\boldsymbol{v}_s\|_{\infty}\right) / \delta\right)^{1 / 2}$
	\end{proof}
	Combining Theorem 1 and Theorem 2, we obtain corollary:
	\begin{corollary}
		The hierarchical control system with inputs human intention $\boldsymbol{v}_h$ and robot intention $\boldsymbol{v}_t$ and state consensus error $\boldsymbol{e}$ is input-state stable and satisfies 
		\begin{equation}
		\|\boldsymbol{e}\|_{\infty} \leq\left(2 N \hat{f}_s\left(\frac{1}{2}\sqrt{\frac{\omega_1}{\omega_0}}\|\boldsymbol{v}_t\|_{\infty}+\frac{1}{2}\sqrt{\frac{\omega_3}{\omega_0}}\|\boldsymbol{v}_h\|_{\infty}\right) / \delta\right)^{1 / 2}.
		\end{equation}
	\end{corollary}

\clearpage
\bibliographystyle{IEEEtran}
\bibliography{IEEEabrv,IEEEexample}

\end{document}